\documentclass[a4paper,11pt]{article}
\usepackage[margin=1truein]{geometry}
\usepackage[hyphens]{url}
\usepackage{graphicx}
\usepackage[round]{natbib}
\usepackage{caption}
\usepackage{amsmath,amsthm,amssymb,mathtools}
\usepackage{thmtools,thm-restate}
\theoremstyle{plain}
\newtheorem{theorem}{Theorem}
\newtheorem{lemma}[theorem]{Lemma}
\newtheorem{proposition}[theorem]{Proposition}
\newtheorem{observation}[theorem]{Observation}
\newtheorem{condition}[theorem]{Condition}
\theoremstyle{definition}
\newtheorem{definition}[theorem]{Definition}
\newtheorem{problem}[theorem]{Problem}
\theoremstyle{remark}

\newtheorem{example}[theorem]{Example}

\usepackage{booktabs}
\usepackage{comment}
\usepackage{authblk}

\usepackage{xcolor}
\usepackage{etoolbox}
\usepackage[linesnumbered,ruled,vlined]{algorithm2e}
\DontPrintSemicolon
\SetInd{0.25em}{0.5em}
\SetVlineSkip{0.0pt}
\SetKwInOut{Input}{Input}
\SetKwInOut{Output}{Output}
\SetKwProg{Func}{Function}{:}{}
\RestyleAlgo{ruled}
\colorlet{commentgray}{black!60!}

\SetCommentSty{mycommfont}
\SetAlCapHSkip{1pt}
\makeatletter
\patchcmd{\@algocf@start}{-1.5em}{0em}{}{} 
\makeatother
\usepackage{caption}

\usepackage{xspace}
\newcommand{\mathtrue}{\textit{true}}
\newcommand{\mathfalse}{\textit{false}}
\newcommand{\scope}{\text{\rm Var}}

\newcommand{\prob}{\text{\rm Pr}}
\newcommand{\vtree}{\mathsf{T}}
\newcommand{\vnode}{\mathsf{v}}
\newcommand{\wnode}{\mathsf{w}}
\newcommand{\pnode}{\mathsf{p}}
\newcommand{\dnode}{\mathsf{d}}
\newcommand{\lcanode}{\mathsf{anc}}
\newcommand{\covfunc}{\mathtt{COV}}
\newcommand{\expfunc}{\mathtt{EXP}}
\newcommand{\adjcovfunc}{\mathtt{ADJCOV}}
\newcommand{\adjexpfunc}{\mathtt{ADJEXP}}
\newcommand{\covvar}{\mathtt{c}}
\newcommand{\expvar}{\mathtt{e}}
\newcommand{\vvarvar}{\mathtt{vv}}
\newcommand{\vexpvar}{\mathtt{ev}}
\newcommand{\adjvarvar}{\mathtt{va}}
\newcommand{\adjexpvar}{\mathtt{ea}}
\newcommand{\resvar}{\mathtt{r}}
\newcommand{\lca}{\mathsf{LCA}}
\newcommand{\vset}{\mathcal{V}}
\newcommand{\bnset}{\mathcal{X}}
\newcommand{\parset}{\mathcal{U}}
\newcommand{\parval}{\mathbf{u}}
\newcommand{\evidence}{\mathbf{x}}
\newcommand{\condit}{\mathbf{c}}
\newcommand{\labelset}{\mathcal{L}}
\newcommand{\labelval}{\mathbf{l}}
\newcommand{\modelset}{\mathcal{A}}
\newcommand{\expect}[1]{\mathrm{E}[ #1 ]}
\newcommand{\variance}[1]{\mathrm{V}[ #1 ]}
\newcommand{\covariance}[2]{\mathrm{Cov}[ #1 , #2 ]}
\newcommand{\VCquery}{\textbf{VC}\xspace}
\newcommand{\CVCquery}{\textbf{CVC}\xspace}
\newcommand{\CTquery}{\textbf{CT}\xspace}
\newcommand{\SEquery}{\textbf{SE}\xspace}
\newcommand{\order}[1]{O( #1 )}

\usepackage{hyperref}
\hypersetup{
  colorlinks=true,
  linkcolor={red!35!black},
  citecolor={green!35!black},
  urlcolor={blue!50!black}
}

\usepackage{datetime}
\date{\vspace{-2.5\baselineskip}}
\newcommand{\email}[1]{\href{mailto:#1}{\nolinkurl{#1}}}

\title{Variance Computation for Weighted Model Counting\\with Knowledge Compilation Approach\footnote{Full version of the paper accepted for AAAI 2026. URL of the published version is: \url{https://doi.org/10.1609/aaai.v40i23.39007}.}}
\author{Kengo Nakamura}
\author{Masaaki Nishino}
\author{Norihito Yasuda}
\affil{Communication Science Laboratories, NTT, Inc., Kyoto, Japan\\ \{kengo.nakamura,masaaki.nishino,norihito.yasuda\}@ntt.com}

\begin{document}

\maketitle

\begin{abstract}
One of the most important queries in knowledge compilation is weighted model counting (WMC), which has been applied to probabilistic inference on various models, such as Bayesian networks.
In practical situations on inference tasks, the model's parameters have uncertainty because they are often learned from data, and thus we want to compute the degree of uncertainty in the inference outcome.
One possible approach is to regard the inference outcome as a random variable by introducing distributions for the parameters and evaluate the \emph{variance} of the outcome.
Unfortunately, the tractability of computing such a variance is hardly known.
Motivated by this, we consider the problem of computing the variance of WMC and investigate this problem's tractability.
First, we derive a polynomial time algorithm to evaluate the WMC variance when the input is given as a structured d-DNNF.
Second, we prove the hardness of this problem for structured DNNFs, d-DNNFs, and FBDDs, which is intriguing because the latter two allow polynomial time WMC algorithms.
Finally, we show an application that measures the uncertainty in the inference of Bayesian networks.
We empirically show that our algorithm can evaluate the variance of the marginal probability on real-world Bayesian networks and analyze the impact of the variances of parameters on the variance of the marginal.
\end{abstract}

\section{Introduction}
\emph{Knowledge compilation} is a technique that represents a propositional formula, a.k.a., a Boolean function, as a compressed and tractable form.
Once Boolean functions are compiled into certain representations, we can solve various queries in polynomial time in the sizes of the representations~\citep{darwiche02kcmap}.
Among various queries, the most prominent one is \emph{weighted model counting} (WMC), which is the problem of counting the (weighted) number of satisfying assignments of a Boolean function.
WMC has been applied to various probabilistic inference tasks on, e.g., Bayesian networks~\citep{chavira08enc,dilkas21wmc}, factor graphs~\citep{choi13factor}, and probabilistic programming~\citep{fierens11plp,holtzen20pp}.

In practical situations, the parameters of such probabilistic models are often obtained by learning from data~\citep{cozman00credal,heckerman08bntutorial}.
When we lack sufficient data, they may suffer from uncertainty.
Perhaps such an uncertainty leads to unreliable inference results.
However, ordinal inference methods (including methods using WMC) disregard uncertainty in parameters.
Thus, we want to compute the degree of uncertainty in the inference outcome when the parameters are imprecise.
A Bayesian statistical approach regards the inference outcome as a random variable by considering the distributions for the parameters and computes the \emph{variance} of the outcome.
For example, for the Bayesian network in Fig.~\ref{fig:smallbn}(a), we consider the variance of the outcome when parameters follow distributions, as in Fig.~\ref{fig:smallbn}(c).
By introducing the expectation and variance of parameters, the outcome's expectation equals the marginal, and we can also obtain its variance.
The computed variance affects the decision-making that depends on the inference outcome; when the computed variance is too large, we should regard the inference result as unreliable.
However, the tractability of computing the variance of the inference outcome remains unknown; although the variance computation is expected to be at least as difficult as the ordinal inference, we do not know the extent of its difficulty.

\begin{figure}[tb]
    \centering
    \includegraphics[keepaspectratio,scale=1.2]{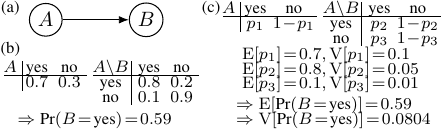}
    \caption{(a) A Bayesian network. (b) Example of ordinal inference, where parameters are fixed. (c) Example of our situation where parameters have variances.}
    \label{fig:smallbn}
\end{figure}

In the applications of inference, the WMC value typically equals the inference outcome.
Thus, motivated by the variance computation of the inference outcome, we consider the query of computing the \emph{variance of WMC value} when the weights associated with Boolean variables have variances.
As explained later, a previous study~\citep{nakamura22variance} treated the variance of WMC value with knowledge compilation in a special case of network analysis.
However, this work is specialized to network analysis and does not consider general WMC tasks.
Moreover, it only uses ordered binary decision diagrams (OBDDs)~\citep{bryant1986graph}, one of the most restricted representations in knowledge compilation.
Therefore, this work hardly reveals the tractability of variance computation including inference tasks.
Thus, we formalize variance computation query for the WMC of general Boolean function and investigate the tractability of it with various knowledge compilation representations.

Our contributions are three-fold.
First, we propose a polynomial time algorithm that computes the variance of WMC of a Boolean function represented as a structured d-DNNF~\citep{pipatsrisawat08structured}.
This result is meaningful since structured d-DNNFs subsume sentential decision diagrams (SDDs)~\citep{darwiche11sdd}, which have been widely used in many applications, as a subset.
Second, we prove that we cannot compute the WMC's variance in polynomial time unless P=NP when the Boolean function is represented as a structured DNNF, a d-DNNF~\citep{darwiche01ddnnf}, or an FBDD~\citep{gergov94fbdd}, all of which are strict supersets of structured d-DNNFs.
The results for d-DNNFs and FBDDs are interesting because the WMC itself can be computed in polynomial time for these representations.
Third, we present an application for the inference of Bayesian networks and show that the variance of the marginal probability can be obtained in polynomial time for a Bayesian network with a constant treewidth.
We also empirically demonstrate the tractability of the proposed algorithm with real-world Bayesian networks and showcase an example of uncertainty analysis on Bayesian networks with variance computation.
Particularly, we demonstrated that we can find parameters of a Bayesian network whose variances have greater impact on the variance of the marginal probability, a useful result for the additional learning of parameters that effectively reduce the uncertainty of the inference.


\section{Related Work}
Knowledge compilation is regarded as a key technique for tackling computationally difficult propositional reasoning tasks.
Thus, as well as the succinctness of representations, the tractability for various operations is the central research subject.
Knowledge compilation map~\citep{darwiche02kcmap}, which summarizes the succinctness and tractability of various representations, have been extended by subsequent studies.
For example, the tractability of standard operations has been studied for recently proposed representations~\citep{illner25wdnnf,onaka25stdasc} and
the tractability of the generalization of WMC such as algebraic model counting (AMC) and two-level AMC (2AMC) was recently investigated~\citep{kiesel22amc,wang24atlas}.
Our study broadens the application of knowledge compilation by proposing a new query related to probabilistic inference, which is a major application of knowledge compilation, and investigating this query's position on the knowledge compilation map.

In probabilistic inference, it is crucial to deal with uncertainty in parameters.
A typical approach to incorporate uncertainty is a fully Bayesian approach, where we regard every parameter as drawn from a distribution, as in the Introduction.
However, to the best of our knowledge, no study has considered the variance of the marginal in a Bayesian network with this approach.
Another line of research for incorporating uncertainty in Bayesian networks is credal networks~\citep{cozman00credal}, where imprecise probabilities are modeled as sets of distributions called credal sets.
Credal networks enable robust inferences by computing the bounds of the marginal probability when the parameters have fluctuated within given bounds.
However, marginal inference for credal networks is NP-hard even for networks with constant treewidth~\citep{decompos05complexity}.
In contrast, our approach can compute the variance of the marginal in polynomial time for networks with constant treewidth.

WMC has also been applied to the reliability analysis on communication networks where links are stochastically failed~\citep{duenas17reliability}.
For this purpose, Boolean function $f^\prime$, which indicates the connectivity in sub-networks, is considered and the reliability equals the WMC of $f^\prime$~\citep{hardy2007knr}.
\citet{nakamura22variance} proposed an algorithm that computes the variance of reliability in polynomial time in the size of the OBDD~\citep{bryant1986graph} representing $f^\prime$ when the existential probability of each link in the network has variance.
We extend their problem setting and algorithm to handle WMC's variance computation of a general Boolean function.
Moreover, we extended their algorithm to work on structured d-DNNFs, a strict superset of OBDDs; here, our algorithm's key technical difference is its management of variable sets and variable decompositions guided by a vtree, as described later.
As a byproduct, we can prove that the variance of network reliability on networks with constant treewidth can be computed in polynomial time.
This theoretically improves the previous result~\citep{nakamura22variance} stating that it can be computed in polynomial time for networks with constant \emph{pathwidth}, since the treewidth subsumes the pathwidth but not vice versa; details are in Appendix~\ref{app:network}.

There exist studies to represent a probability distribution of a random variable $X$ as a tractable circuit in a spirit of knowledge compilation:
probabilistic circuits~\citep{choi20pc} represent probability mass functions, while probabilistic generating circuits~\citep{zhang21pgc} and characteristic circuits~\citep{yu23character} represent probability generating and characteristic functions.
These circuits admit polytime moment computation, including the variance, of the random variable $X$ under certain structural restrictions.
In contrast, our work regards the \emph{probability} $\prob(X=a)$ as a random variable and computes the variance of it, where $X$ is a random variable appearing in, e.g., Bayesian networks.
To derive the variance of the inference outcome, our work is needed because we currently have no approach to compute the variance of the probability value seen as a random variable with probabilistic circuits.

\section{Preliminaries}
A \emph{Boolean function} takes a set of Boolean variables each valued $\mathtrue$ or $\mathfalse$ as an input and outputs either $\mathtrue$ or $\mathfalse$.
An \emph{assignment} $a$ on variable set $\vset$ is a mapping $\vset\rightarrow\{\mathtrue,\mathfalse\}$.
Assignment $a$ is called a \emph{model} of Boolean function $f$ if $f$ is evaluated to $\mathtrue$ under $a$.

  A rooted directed acyclic graph is called a \emph{negation normal form} (\emph{NNF}) if the leaf nodes are labeled with $\mathtrue$, $\mathfalse$, $x$, or $\neg x$, where $x$ is a Boolean variable, and the internal nodes are labeled with either $\wedge$ or $\vee$.
  The size of the NNF is defined as the number of arcs.
  For node $\alpha$ of an NNF, Boolean function $f_\alpha$ represented by $\alpha$ is defined as follows.
  For leaf node $\alpha$, $f_\alpha=\alpha$; $\mathtrue$ and $\mathfalse$ stand for identity functions that always evaluate to $\mathtrue$ and $\mathfalse$.
  For internal node $\alpha$, let $\alpha_1,\ldots,\alpha_k$ be the child nodes of $\alpha$.
  If $\alpha$ is a $\wedge$-node, $f_\alpha=\bigwedge_jf_{\alpha_j}$.
  If $\alpha$ is a $\vee$-node, $f_\alpha=\bigvee_jf_{\alpha_j}$.
  The Boolean function represented by an NNF is that represented by its root node.
We often abuse a symbol for NNF node $\alpha$ to represent the whole NNF rooted at $\alpha$.

Next, we define several restrictions on NNFs, which induce subsets of NNFs.
For NNF node $\alpha$, let $\scope(\alpha)$ be the set of Boolean variables that appear as the labels of the descendant nodes of $\alpha$, called the \emph{scope} of $\alpha$.
In the following, let $\alpha_1,\ldots,\alpha_k$ be the child nodes of internal node $\alpha$.
\begin{definition}
  An NNF is called \emph{decomposable} if every $\wedge$-node $\alpha$ satisfies $\scope(\alpha_i)\cap\scope(\alpha_j)=\emptyset$ for any $i\neq j$.
  An NNF is called \emph{deterministic} if every $\vee$-node $\alpha$ satisfies $f_{\alpha_i}\wedge f_{\alpha_j}=\mathfalse$ for any $i\neq j$.
  An NNF is called \emph{decision} if every $\vee$-node only appears in the form: $(x\wedge\alpha)\vee(\neg x\wedge\beta)$, where $x,\neg x$ are leaf nodes.
\end{definition}
A \emph{d-DNNF} is a decomposable and deterministic NNF, and \emph{FBDD} is a decomposable and decision NNF with the following additional restriction; for every $\vee$-node, $\alpha,\beta$ in the decision property must be either a leaf node or a $\vee$-node.
We also define structured decomposability as follows.
\begin{definition}
  A \emph{vtree} $\vtree$ on variable set $\vset$ is a rooted binary tree, where each leaf node is labeled with a Boolean variable in $\vset$ and each internal node $\vnode$ has exactly two child nodes $\vnode^l,\vnode^r$.
  Here, any Boolean variable $x\in\mathcal{V}$ must appear as a label exactly once.
  To distinguish them from NNF nodes, we call the nodes of a vtree a \emph{vnode}.
  The \emph{scope} $\scope(\vnode)$ of vnode $\vnode$ is the set of the labels of the descendants of $\vnode$.
  For NNF node $\alpha$, its \emph{decomposition vnode} $\dnode(\alpha)$ of vtree $\vtree$ is the deepest vnode $\vnode$ in $\vtree$ satisfying $\scope(\alpha)\subseteq\scope(\vnode)$. 
\end{definition}
\begin{definition}
  \label{def:structured}
  We say an NNF \emph{respects vtree $\vtree$} if every $\wedge$-node $\alpha$ has exactly two child nodes $\alpha^l,\alpha^r$ and they satisfy $\scope(\alpha^l)\subseteq\scope(\vnode^l)$ and $\scope(\alpha^r)\subseteq\scope(\vnode^r)$ for some vnode $\vnode$ of $\vtree$.
  An NNF is called \emph{structured decomposable} if it respects some vtree.
\end{definition}
A \emph{structured DNNF} (\emph{st-DNNF}) is a structured decomposable NNF.
A \emph{structured d-DNNF} (\emph{st-d-DNNF}) is a structured decomposable and deterministic NNF.

\begin{figure}
    \centering
    \includegraphics[keepaspectratio,scale=1.2]{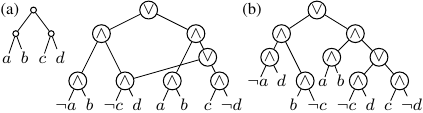}
    \caption{(a) A vtree and an st-d-DNNF. (b) A d-DNNF that is not structured decomposable.}
    \label{fig:stddnnf}
\end{figure}

\begin{example}
  \label{ex:figure}
  Let $f=(\neg a\wedge b\wedge\neg c\wedge d)\vee(a\wedge b\wedge\neg c\wedge d)\vee(a\wedge b\wedge c\wedge\neg d)$.
  Fig.~\ref{fig:stddnnf}(a) depicts an st-d-DNNF of $f$ and the respected vtree.
  Fig.~\ref{fig:stddnnf}(b) is a d-DNNF of $f$; however, it is not structured decomposable because the left child of the root decomposes the variables into $\{a,d\}$ and $\{b,c\}$ while the right child decomposes them into $\{a,b\}$ and $\{c,d\}$.
\end{example}

Here, we assume that, for any $\wedge$-node $\alpha$ of an st-d-DNNF with child nodes $\alpha^l,\alpha^r$, $\scope(\alpha^l)\neq\emptyset$ and $\scope(\alpha^r)\neq\emptyset$.
We can easily transform an st-d-DNNF to satisfy the above assumption.
If $\scope(\alpha^l)=\emptyset$, $f_{\alpha^l}$ is either $\mathtrue$ or $\mathfalse$.
When $f_{\alpha^l}=\mathtrue$, we can replace $\alpha$ with $\alpha^r$; i.e., we eliminate $\alpha$ and redirect the incoming arcs of $\alpha$ to $\alpha^r$.
Otherwise, we can replace $\alpha$ with $\mathfalse$.
We can perform the same transformation when $\scope(\alpha^r)=\emptyset$.
Under this assumption, the vnode $\vnode$ appeared in Definition~\ref{def:structured} is determined as $\vnode=\dnode(\alpha)$.
This can be proved as follows.
We have $\scope(\alpha)=\scope(\alpha^l)\cup\scope(\alpha^r)\subseteq\scope(\vnode^l)\cup\scope(\vnode^r)=\scope(\vnode)$.
Also, we have $\scope(\alpha)\nsubseteq\scope(\vnode^l)$ following from $\scope(\alpha)\setminus\scope(\vnode^l)=\scope(\alpha^r)\neq\emptyset$. Similarly, $\scope(\alpha)\nsubseteq\scope(\vnode^r)$.
Therefore, $\vnode$ is the deepest vnode such that $\scope(\alpha)\subseteq\scope(\vnode)$.

\section{Variance of Weighted Model Counting}
We first define the WMC.
We denote the set of models of $f$ on variable set $\vset$ by $\modelset_f^\vset$.
For each variable $x$ in variable set $\vset$, we assign \emph{positive weight} $P_x$ and \emph{negative weight} $N_x$.
Then we define the WMC $W_f^\vset$ of $f$ on variable set $\vset$ by
\begin{equation}
  W_f^\vset \coloneqq\sum_{a\in\modelset_f^\vset}W_a^\vset,\quad
  W_a^\vset\coloneqq\ \smashoperator{\prod_{\substack{x\in\mathcal{V}\\a(x)=\mathtrue
  }}}\ P_x\cdot\ \smashoperator{\prod_{\substack{x\in\mathcal{V}\\a(x)=\mathfalse}}}\ N_x. \label{eq:wmc}
\end{equation}
Note that the value of $W_f^\vset$ changes by modifying $\vset$.
Thus, when considering the WMC, we must care about the variable set.
If $\vset$ is clear from the context, we omit the superscripts.

In existing studies, $P_x$ and $N_x$ are given as real values without uncertainty.
In this paper, for every variable $x\in\mathcal{V}$, $P_x$ and $N_x$ are regarded as random variables with bounded expectation and variance.
Then, WMC $W_f$, defined by (\ref{eq:wmc}), is also a random variable with bounded expectation and variance.
This virtually considers the variance of the inference outcome in the applications since the WMC value typically equals the outcome, as described in Introduction.

We assume that $(P_x,N_x)$ and $(P_y,N_y)$ are independent for $x\neq y$, while $P_x$ and $N_x$ for the same $x$ are not necessarily independent.
Then the expectation $\expect{W_f}$ is equivalent to the ordinal WMC:
$\expect{W_f^\vset}=\sum_{a\in\modelset_f^\vset}\expect{W_a^\vset}=\sum_{a\in\modelset_f^\vset}\prod_{x\in\vset:a(x)=\mathtrue}\expect{P_x}\cdot\prod_{x\in\vset:a(x)=\mathfalse}\expect{N_x}$.
This assumption is reasonable for some applications, and later we slightly relax it for a specific application; see the Application section as well as Appendices~\ref{app:multivalued} and \ref{app:network}.
Now we formally define the variance computation queries.
\begin{problem}
  \label{prob:variance}
  We are given expectations $\mu_{P_x},\mu_{N_x}$ and variances $\sigma_{P_x}^2,\sigma_{N_x}^2$ of $P_x,N_x$ and covariance $\sigma_{P_xN_x}$ of $P_x$ and $N_x$ for every $x\in\vset$.
  We define \emph{variance computation} query \VCquery as the computation of variance $\variance{W_f^\vset}$ of WMC of input Boolean function $f$.
  As a related one, we define \emph{covariance computation} query \CVCquery as the computation of covariance $\covariance{W_f^\vset}{W_g^\vset}$ of WMCs of input Boolean functions $f,g$.
\end{problem}

We have $\covariance{W_f}{W_g}=\sum_{a\in\modelset_f}\sum_{b\in\modelset_g}\covariance{W_a}{W_b}$, each term of which can be computed in $\order{|\vset|}$ time. Thus, if the models of Boolean functions are explicitly enumerated, we can compute $\variance{W_f}=\covariance{W_f}{W_f}$ in $\order{|\vset||\modelset_f|^2}$ time and $\covariance{W_f}{W_g}$ in $\order{|\vset||\modelset_f||\modelset_g|}$ time.
\begin{example}
  \label{ex:variance}
  Let $f$ be the Boolean function in Example~\ref{ex:figure}.
  Let $\mu_{P_x}=\mu$, $\mu_{N_x}=1-\mu$, $\sigma_{P_x}^2=\sigma_{N_x}^2=\sigma^2$, and $\sigma_{P_xN_x}=-\sigma^2$ for any $x\in\vset=\{a,b,c,d\}$.
  Then $W_f=N_aP_bN_cP_d+P_aP_bN_cP_d+P_aP_bP_cN_d$, and thus $\expect{W_f}=\mu^2-\mu^4$.
  Similarly, $\variance{W_f}=(2\mu^2-2\mu^3-2\mu^4+4\mu^6)\sigma^2+(1-2\mu+2\mu^2+6\mu^4)\sigma^4+(2+4\mu^2)\sigma^6+\sigma^8$.
\end{example}

However, since $|\modelset_f|$ and $|\modelset_g|$ are generally exponential in $|\vset|$, this solution causes a prohibitively long running time.
Therefore, we consider how to solve these queries when Boolean functions are represented as NNFs.

\section{Tractability Results}
The goal of this section is to prove the following theorem.
\begin{theorem}
  \label{thm:tractable}
  When $f,g$ are given as st-d-DNNFs $\alpha,\beta$ respecting the same vtree, \CVCquery can be solved in $\order{|\alpha||\beta|+|\vset|^2}$ time.
  Thus, when $f$ is given as an st-d-DNNF $\alpha$, \VCquery can be solved in $\order{|\alpha|^2+|\vset|^2}$ time.
\end{theorem}

We first introduce some fundamental formulas that are frequently used.
Given random variables $A,B,C,X,Y$, suppose that $(A,B)$ and $(X,Y)$ are independent. Then,
\begin{align}
  \covariance{A+B}{C} & =\covariance{A}{C}+\covariance{B}{C}, \label{eq:sumcov}\\
  \covariance{AX}{BY} & =\covariance{A}{B}\covariance{X}{Y} 
  + \covariance{A}{B}\expect{X}\expect{Y} 
  + \expect{A}\expect{B}\covariance{X}{Y}. \label{eq:prodcov}
\end{align}
Eq. (\ref{eq:sumcov}) is a well-known formula derived from the linearity of covariances.
Eq. (\ref{eq:prodcov}) is analogous to the formula for the variance of the product of independent random variables; the proof of (\ref{eq:prodcov}) can be found in~\citep{nakamura22variance}.

Using these formulas, we design an algorithm to compute $\covariance{W_{f_\alpha}}{W_{f_\beta}}$ for given st-d-DNNFs $\alpha,\beta$ respecting the same vtree.
The proposed algorithm computes $\covariance{W_{f_\alpha}}{W_{f_\beta}}$ by recursively decomposing it into the sums and products of $\covariance{W_{f_{\alpha^\prime}}}{W_{f_{\beta^\prime}}}$s, where $\alpha^\prime,\beta^\prime$ are the child nodes of $\alpha,\beta$.
To avoid redundant recursive calls, the value of $\covariance{W_{f_{\alpha^\prime}}}{W_{f_{\beta^\prime}}}$ is cached once it is computed.
However, since WMC value $W_f^\vset$ is altered by changing variable set $\vset$, we must track the variable set in decomposing the covariance.
We manage the variable set by fully using the vtree.
More specifically, let $\lcanode=\lca(\dnode(\alpha),\dnode(\beta))$ be the \emph{lowest common ancestor} (\emph{LCA}) of $\dnode(\alpha)$ and $\dnode(\beta)$, which is the deepest vnode $\vnode$ such that it is the ancestor of both $\dnode(\alpha)$ and $\dnode(\beta)$.
Our algorithm recursively computes $\covariance{W_{f_\alpha}^\vset}{W_{f_\beta}^\vset}$, where $\vset=\scope(\lcanode)$.
In the following, for convenience, we define $\dnode(\mathtrue)=\dnode(\mathfalse)=\bot$, which is an imaginary vnode satisfying $\scope(\bot)=\emptyset$, $\lca(\bot,\bot)=\bot$, and $\lca(\vnode,\bot)=\vnode$ for any other vnode $\vnode$.
In other words, $\bot$ is a vnode that is a descendant of any other vnodes.

\subsection{Decomposition Lemmas}
To derive the algorithm, we must determine how the covariance is decomposed into the covariances of child nodes.
We derive decomposition formulas by conducting a comprehensive case analysis:
(I) $\dnode(\alpha)$ and $\dnode(\beta)$ have no ancestor-descendant relation, (II) $\dnode(\alpha)$ is an ancestor of $\dnode(\beta)$ and $\alpha$ is a $\vee$-node, and (III) $\dnode(\alpha)$ is an ancestor of $\dnode(\beta)$ and $\alpha$ is a $\wedge$-node.
Here, (II) and (III) allow $\dnode(\alpha)=\dnode(\beta)$.
Note that when $\dnode(\beta)$ is an ancestor of $\dnode(\alpha)$, we can swap $\alpha$ and $\beta$ to satisfy (II) or (III).
In the following, we derive decomposition formulas for each case.

In case (I), i.e., both $\lcanode\neq\dnode(\alpha)$ and $\lcanode\neq\dnode(\beta)$ hold, the following decomposition holds by considering how $f_\alpha$ and $f_\beta$ can be represented on variable set $\scope(\lcanode)$.
\begin{lemma}
  \label{lem:norelation}
  In case (I), by letting $\vset\coloneqq\scope(\lcanode)$, $\vset^l\coloneqq\scope(\lcanode^l)$, and $\vset^r\coloneqq\scope(\lcanode^r)$, we have
  \begin{align}
    \covariance{W_{f_\alpha}^\vset}{W_{f_\beta}^\vset} & =  \covariance{W_{f_\alpha}^{\vset^l}}{W_{\mathtrue}^{\vset^l}}\covariance{W_{\mathtrue}^{\vset^r}}{W_{f_\beta}^{\vset^r}} \nonumber \\
  & \ \ + \covariance{W_{f_\alpha}^{\vset^l}}{W_{\mathtrue}^{\vset^l}} \expect{W_{\mathtrue}^{\vset^r}}\expect{W_{f_\beta}^{\vset^r}} \label{eq:norelation} \\
  & \ \ + \expect{W_{f_\alpha}^{\vset^l}}\expect{W_{\mathtrue}^{\vset^l}}\covariance{W_{\mathtrue}^{\vset^r}}{W_{f_\beta}^{\vset^r}}. \nonumber
  \end{align}
\end{lemma}
\begin{proof}
  Since $\scope(\alpha)\subseteq\vset^l$, $\vset^l\cup\vset^r=\vset$, and $\vset^l\cap\vset^r=\emptyset$, $f_\alpha$ on variable set $\vset$ can be represented as $f_\alpha\wedge\mathtrue^{\vset^r}$, where $\mathtrue^{\vset^r}$ is a $\mathtrue$ function on variable set $\vset^r$.
  Thus, we have $W_{f_\alpha}^\vset=W_{f_\alpha}^{\vset^l}W_{\mathtrue}^{\vset^r}$.
  Similarly, $W_{f_\beta}^\vset=W_{\mathtrue}^{\vset^l}W_{f_\beta}^{\vset^r}$.
  Since $(W_{f_\alpha}^{\vset^l},W_{\mathtrue}^{\vset^l})$ and $(W_{\mathtrue}^{\vset^r},W_{f_\beta}^{\vset^r})$ are independent, the lemma follows from (\ref{eq:prodcov}).
\end{proof}

In case (II), we have the following decomposition.
\begin{lemma}
  \label{lem:ordecomp}
  In case (II), by letting $\vset\coloneqq\scope(\lcanode)$ and $\alpha_1,\ldots,\alpha_k$ be the child nodes of $\alpha$, we have
  \begin{equation}
    \textstyle \covariance{W_{f_\alpha}^\vset}{W_{f_\beta}^\vset}=\sum_{j=1}^{k}\covariance{W_{f_{\alpha_j}}^\vset}{W_{f_\beta}^\vset}. \label{eq:ordecomp}
  \end{equation}
\end{lemma}
\begin{proof}
  By determinism of $f_\alpha=\bigvee_{j=1}^{k}f_{\alpha_j}$, we have $W_{f_{\alpha}}^\vset=\sum_{j=1}^{k}\!W_{f_{\alpha_j}}^\vset$.
  Eq.~(\ref{eq:ordecomp}) follows by recursively applying (\ref{eq:sumcov}).
\end{proof}

In case (III), two child nodes $\alpha^l,\alpha^r$ satisfy $\scope(\alpha^l)\subseteq\scope(\lcanode^l)$ and $\scope(\alpha^r)\subseteq\scope(\lcanode^r)$ by structured decomposability.
This leads to the following decomposition.
\begin{lemma}
  \label{lem:anddecomp}
  In case (III), suppose $f_\beta$ can be decomposed as $f_{\beta}^\prime\wedge f_{\beta}^{\prime\prime}$, where $\scope(f_{\beta}^\prime)\subseteq\scope(\lcanode^l)$ $\eqqcolon\vset^l$ and $\scope(f_{\beta}^{\prime\prime})\subseteq\scope(\lcanode^r)\eqqcolon\vset^r$.
  Then, by letting $\vset\coloneqq\scope(\lcanode)$,
  \begin{align}
    \covariance{W_{f_\alpha}^\vset}{W_{f_\beta}^\vset} & =  \covariance{W_{f_{\alpha^l}}^{\vset^l}}{W_{f_\beta^\prime}^{\vset^l}}\covariance{W_{f_{\alpha^r}}^{\vset^r}}{W_{f_\beta^{\prime\prime}}^{\vset^r}} \nonumber \\
  & \ \ + \covariance{W_{f_{\alpha^l}}^{\vset^l}}{W_{f_{\beta}^\prime}^{\vset^l}} \expect{W_{f_{\alpha^r}}^{\vset^r}}\expect{W_{f_{\beta}^{\prime\prime}}^{\vset^r}} \label{eq:anddecomp} \\
  & \ \ + \expect{W_{f_{\alpha^l}}^{\vset^l}}\expect{W_{f_{\beta}^\prime}^{\vset^l}}\covariance{W_{f_{\alpha^r}}^{\vset^r}}{W_{f_\beta^{\prime\prime}}^{\vset^r}}. \nonumber
  \end{align}
\end{lemma}
\begin{proof}
  We have $W_{f_\alpha}^\vset=W_{f_{\alpha^l}}^{\vset^l}W_{f_{\alpha^r}}^{\vset^r}$ and $W_{f_\beta}^\vset=W_{f_\beta^\prime}^{\vset^l}W_{f_\beta^{\prime\prime}}^{\vset^r}$, where $(W_{f_{\alpha^l}}^{\vset^l},W_{f_\beta^\prime}^{\vset^l})$ and $(W_{f_{\alpha^r}}^{\vset^r},W_{f_\beta^{\prime\prime}}^{\vset^r})$ are independent.
  The lemma follows from (\ref{eq:prodcov}).
\end{proof}

We can decompose $f_\beta=f_\beta^\prime\wedge f_\beta^{\prime\prime}$ for the following cases.
If $\lcanode\neq\dnode(\beta)$, either $\scope(\dnode(\beta))\subseteq\vset^l$ or $\scope(\dnode(\beta))\subseteq\vset^r$ holds.
We can take $(f_\beta^\prime,f_\beta^{\prime\prime})=(f_\beta,\mathtrue)$ for the former and $(f_\beta^\prime,f_\beta^{\prime\prime})=(\mathtrue,f_\beta)$ for the latter.
If $\lcanode=\dnode(\beta)$ and $\beta$ is a $\wedge$-node, child nodes $\beta^l,\beta^r$ satisfy $\scope(\beta^l)\subseteq\vset^l$ and $\scope(\beta^r)\subseteq\vset^r$ by structured decomposability.

\subsection{Procedure and Complexity}
We can recursively decompose $\covariance{W_{f_\alpha}}{W_{f_\beta}}$ into the covariances and expectations of the WMCs of child nodes with Lemmas~\ref{lem:norelation}--\ref{lem:anddecomp}.
The base cases of the recursion, e.g., the case where both are literals with the same Boolean variable, can be resolved using the input (co)variances $\sigma_{P_x}^2,\sigma_{N_x}^2,\sigma_{P_xN_x}$ of weights.
Also, we pre-compute $\expect{W_{f_\gamma}}$ for every node $\gamma$ in st-d-DNNFs $\alpha,\beta$.
Since this procedure is identical to a standard one for computing WMC with st-d-DNNFs, the details of computing expectations are in Appendix~\ref{app:preprocessing}.

Algorithm~\ref{alg:prod} is the proposed covariance computation algorithm.
This algorithm outputs a pair $(\lcanode,\covariance{W_{f_\alpha}^\vset}{W_{f_\beta}^\vset})$, where $\vset=\scope(\lcanode)$.
We cache the output for $\covariance{\alpha}{\beta}$ in $\covvar[\alpha,\beta]$ once computed.
$\expvar[\gamma]$ stores a pair of $\vnode=\dnode(\gamma)$ and $\expect{W_{f_\gamma}^{\scope(\vnode)}}$.
As stated above, these can be pre-computed with a standard WMC algorithm; we defer the details to Appendix~\ref{app:preprocessing}.
Lines 3 and 4 deal with the base cases and lines 5--10 use Lemma~\ref{lem:norelation}.
Lines 11--16 deal with the remaining base cases involving literals.
Lines 19 and 20 use Lemma~\ref{lem:ordecomp} and lines 21--31 use Lemma~\ref{lem:anddecomp}.
To ensure that $f_\beta$ can be decomposed into $f_\beta^\prime\wedge f_\beta^{\prime\prime}$ as in Lemma~\ref{lem:anddecomp}, $\alpha,\beta$ are swapped in line~18, if needed.

\begin{algorithm}[!t]
\caption{$\covfunc(\alpha,\beta)$: computing $\covariance{W_{f_\alpha}}{W_{f_\beta}}$}
\label{alg:prod}
{\footnotesize
\Input{Two st-d-DNNFs $\alpha, \beta$ respecting the same vtree}
\Output{Pair of $\lcanode=\lca(\dnode(\alpha),\dnode(\beta))$ and $\covariance{W_{f_\alpha}^\vset}{W_{f_\beta}^\vset}$ ($\vset=\scope(\lcanode)$)}
\lIf(\tcp*[f]{Cache for $\covfunc(\alpha,\beta)$}){$\covvar[\alpha,\beta]\neq \textrm{null}$}{\Return $\covvar[\alpha,\beta]$}
$\lcanode\gets \lca(\dnode(\alpha),\dnode(\beta))$\;
\lIf {$\alpha=\mathfalse$ or $\beta=\mathfalse$}{\Return $(\lcanode,0)$}
\lIf {$\alpha=\mathtrue$ and $\beta=\mathtrue$}{\Return $(\bot,0)$}
\uIf(\tcp*[f]{Let $\scope(\alpha)\subseteq\scope(\lcanode^l)$ and $\scope(\beta)\subseteq\scope(\lcanode^r)$}){$\lcanode\neq\dnode(\alpha)$ and $\lcanode\neq\dnode(\beta)$}{
  $\mathtt{el}\leftarrow\adjexpfunc(\lcanode^l,\expvar[\alpha])\cdot\adjexpfunc(\lcanode^l,\expvar[\mathtrue])$\;
  $\mathtt{er}\leftarrow\adjexpfunc(\lcanode^r,\expvar[\mathtrue])\cdot\adjexpfunc(\lcanode^r,\expvar[\beta])$\;
  $\mathtt{cl}\leftarrow\adjcovfunc(\lcanode^l,\covfunc(\alpha,\mathtrue),\expvar[\alpha],\expvar[\mathtrue])$\;
  $\mathtt{cr}\leftarrow\adjcovfunc(\lcanode^r,\covfunc(\mathtrue,\beta),\expvar[\mathtrue],\expvar[\beta])$\;
  $\resvar\leftarrow\mathtt{cl}\cdot\mathtt{cr}+\mathtt{cl}\cdot\mathtt{er}+\mathtt{el}\cdot\mathtt{cr}$ \tcp*{Eq. (\ref{eq:norelation})}
}
\uElseIf {$\alpha,\beta$ are both leaf nodes}{
  \lIf{$(\alpha,\beta)=(\mathtrue,x),(x,\mathtrue)$}{$\resvar\leftarrow \sigma^2_{P_x}+\sigma_{P_xN_x}$}
  \lElseIf{$(\alpha,\beta)\!=\!(\mathtrue,\neg x),(\neg x,\mathtrue)$}{$\resvar\!\leftarrow\!\sigma^2_{N_x}\!+\!\sigma_{P_xN_x}$}
  \lElseIf{$\alpha=\beta=x$}{$\resvar\leftarrow\sigma^2_{P_x}$}
  \lElseIf{$\alpha=\beta=\neg x$}{$\resvar\leftarrow\sigma^2_{N_x}$}
  \lElse(\tcp*[f]{$(\alpha,\beta)=(x,\neg x),(\neg x,x)$}){$\resvar\leftarrow\sigma_{P_xN_x}$}
}
\Else{
  Swap $\alpha,\beta$ if (i) $\lcanode=\dnode(\beta)\neq\dnode(\alpha)$ or (ii) $\dnode(\alpha)=\dnode(\beta)$ and only $\beta$ is a $\vee$-node\;
  \uIf(\tcp*[f]{$\alpha_1,\ldots,\alpha_k$: the child nodes of $\alpha$}){$\alpha$ is a $\vee$-node}{
    $\resvar\!\leftarrow\!\sum_{j=1}^{k}\adjcovfunc(\lcanode,\covfunc(\alpha_i,\beta),\expvar[\alpha_i],\expvar[\beta])$\tcp*{Eq. (\ref{eq:ordecomp})}
  }
  \Else(\tcp*[f]{$\alpha$ is a $\wedge$-node}){
    $\alpha^l,\alpha^r\leftarrow(\text{child nodes of $\alpha$})$ s.t. $\scope(\alpha^l)\subseteq\scope(\lcanode^l)$ and $\scope(\alpha^r)\subseteq\scope(\lcanode^r)$\;
    \lIf{$\scope(\dnode(\beta))\subseteq\scope(\lcanode^l)$}{$\beta^l\leftarrow\beta$, $\beta^r\leftarrow\mathtrue$}
    \lElseIf{$\scope(\dnode(\beta))\subseteq\scope(\lcanode^r)$}{$\beta^l\leftarrow\mathtrue$, $\beta^r\leftarrow\beta$}
    \Else(\tcp*[f]{$\dnode(\alpha)=\dnode(\beta)$; thus $\beta$ is a $\wedge$-node due to line 18}){
      $\beta^l,\beta^r\leftarrow(\text{child nodes of $\beta$})$ s.t. $\scope(\beta^l)\subseteq\scope(\lcanode^l)$ and $\scope(\beta^r)\subseteq\scope(\lcanode^r)$}
    $\mathtt{el}\leftarrow\adjexpfunc(\lcanode^l,\expvar[\alpha^l])\cdot\adjexpfunc(\lcanode^l,\expvar[\beta^l])$\;
    $\mathtt{er}\leftarrow\adjexpfunc(\lcanode^r,\expvar[\alpha^r])\cdot\adjexpfunc(\lcanode^r,\expvar[\beta^r])$\;
    $\mathtt{cl}\leftarrow\adjcovfunc(\lcanode^l,\covfunc(\alpha^l,\beta^l),\expvar[\alpha^l],\expvar[\beta^l])$\;
    $\mathtt{cr}\leftarrow\adjcovfunc(\lcanode^r,\covfunc(\alpha^r,\beta^r),\expvar[\alpha^r],\expvar[\beta^r])$\;
    $\resvar\leftarrow\mathtt{cl}\cdot\mathtt{cr}+\mathtt{cl}\cdot\mathtt{er}+\mathtt{el}\cdot\mathtt{cr}$ \tcp*{Eq. (\ref{eq:anddecomp})}
  }
}
\Return $\covvar[\alpha,\beta]\leftarrow (\lcanode,\resvar)$
}
\end{algorithm}

We must care about the variable set during the computation.
For this purpose, we implement two auxiliary functions $\adjexpfunc$ and $\adjcovfunc$.
$\adjexpfunc$ receives vnode $\wnode$ and $\expvar[\alpha^\prime]$, where $\scope(\dnode(\alpha^\prime))\subseteq\scope(\wnode)$, and returns $\expect{W_{f_{\alpha^\prime}}^{\scope(\wnode)}}$.
$\adjcovfunc$ receives vnode $\wnode$, the output of $\covfunc(\alpha^\prime,\beta^\prime)$, $\expvar[\alpha^\prime]$, and $\expvar[\beta^\prime]$, where $\scope(\lca(\dnode(\alpha^\prime),\dnode(\beta^\prime)))\subseteq\scope(\wnode)$, and returns $\covariance{W_{f_{\alpha^\prime}}^{\scope(\wnode)}}{W_{f_{\beta^\prime}}^{\scope(\wnode)}}$.
Using these functions, we adjust the variable sets.
With a preprocessing taking $\order{|\vset|^2}$ time, these functions can be computed in constant time; see Appendix~\ref{app:preprocessing}.
The correctness of Algorithm~\ref{alg:prod}, i.e., that $\covfunc(\alpha,\beta)$ returns $\covariance{W_{f_\alpha}^\vset}{W_{f_\beta}^\vset}$, follows from the fact that cases (I), (II), and (III) are comprehensive and recursive decomposition follows Lemmas~\ref{lem:norelation}--\ref{lem:anddecomp} for each case.
We now move to the proof of Theorem~\ref{thm:tractable}.
\begin{proof}[Proof of Theorem~\ref{thm:tractable}]
  Preprocessing requires $\order{|\vset|^2}$ time, and computing expectations takes $\order{|\alpha|+|\beta|}$ time.
  Computing the LCA (line~2) needs $\order{1}$ time with a data structure that is built in $\order{|\vset|}$ time~\citep{bender00lca}.
  Thus, other than recursion, $\covfunc(\alpha^\prime,\beta^\prime)$ requires at most $\order{k_{\alpha^\prime}k_{\beta^\prime}}$ time, where $k_{\alpha^\prime},k_{\beta^\prime}$ are the number of child nodes of $\alpha^\prime,\beta^\prime$.
  Since the answer is cached in $\covvar[\alpha^\prime,\beta^\prime]$ once $\covfunc(\alpha^\prime,\beta^\prime)$ is computed, the overall complexity of $\covfunc(\alpha,\beta)$ is bounded by $\order{|\alpha||\beta|}$.
\end{proof}

We finally give a brief note on the assumption that two st-d-DNNFs share the same vtree in solving \CVCquery query.
Such an assumption is also imposed on some queries that take multiple st-d-DNNFs as an input~\citep{pipatsrisawat08structured}; e.g., sentential entailment and bounded conjunction defined in~\citep{darwiche02kcmap}.
Although we do not prove the tractability of \CVCquery for the case where two st-d-DNNFs do not respect the same vtree, we believe it is intractable because st-d-DNNFs do not admit polytime sentential entailment unless P=NP when they do not share the vtree.
Note that, for \VCquery, such an assumption is not imposed because we have a single input st-d-DNNF for \VCquery.

\section{Intractability Results}
The goal of this section is to prove the following result.
\begin{theorem}
  \label{thm:intractability}
  When $f,g$ are given as st-DNNFs, d-DNNFs, or FBDDs, \CVCquery is intractable, i.e., it cannot be solved in polynomial time unless P=NP.
  When $f$ is given as an st-DNNF, a d-DNNF, or an FBDD, \VCquery is intractable.
\end{theorem}

We prove this by first introducing some queries from the knowledge compilation map~\citep{darwiche02kcmap}.
\begin{problem}
  \label{prob:others}
  Given Boolean function $f$, \emph{model counting} query \CTquery computes the number of models of $f$, i.e., $|\modelset_f^\vset|$.
  Given Boolean functions $f,g$, \emph{sentential entailment} query \SEquery asks whether $f\models g$, i.e., $\modelset_f^\vset\subseteq\modelset_g^\vset$.
\end{problem}
We now show a polynomial time reduction from the \CTquery and \SEquery queries to the \VCquery and \CVCquery queries.
It is known that \CTquery is intractable when $f$ is given as an st-DNNF~\citep{pipatsrisawat08structured}.
It is also known that \SEquery is intractable when $f,g$ are given as d-DNNFs or FBDDs~\citep{darwiche02kcmap}.
Thus, the existence of the above reduction indicates the intractability of \VCquery and \CVCquery queries with such representations.

Let $n\coloneqq|\vset|$. The key lemmas are as follows.
\begin{lemma}
  \label{lem:assignment}
  Let $\mu_{P_x}=\mu_{N_x}=1$, $\sigma_{P_x}^2=\sigma_{N_x}^2=3$, and $\sigma_{P_xN_x}=-1$ for every variable $x\in\vset$.
  Then, for any assignment $a$ of $\vset$, $\variance{W_a}=4^n-1$.
  In addition, for any assignments $a,b\ (a\neq b)$ of $\vset$, $\covariance{W_a}{W_b}=-1$.
\end{lemma}
\begin{proof}
  We can decompose $W_a^\vset=Q^a_xW_a^{\vset\setminus\{x\}}$, where $Q^a_x=P_x$ if $a(x)=\mathtrue$ or $Q^a_x=N_x$ otherwise.
  Similar decomposition can be derived for $W_b^\vset$.
  Thus, by (\ref{eq:prodcov}),
  \begin{align}
    \covariance{W_a^\vset}{W_b^\vset} & =
    (\covariance{Q^a_x}{Q^b_x}+\expect{Q^a_x}\expect{Q^b_x})\covariance{W_a^{\vset\setminus\{x\}}}{W_b^{\vset\setminus\{x\}}} \nonumber \\
    &\ \ +\covariance{Q^a_x}{Q^b_x}\expect{W_a^{\vset\setminus\{v\}}}\expect{W_b^{\vset\setminus\{x\}}}. \label{eq:lemassignment}
  \end{align}
  Here, $\expect{Q^a_x}=\expect{Q^b_x}=\expect{W_a^{\vset\setminus\{x\}}}=\expect{W_b^{\vset\setminus\{x\}}}=1$ because $\mu_{P_x}=\mu_{N_x}=1$ for any $x\in\vset$.
  When $a=b$, (\ref{eq:lemassignment}) becomes $\variance{W_a^\vset}=4\variance{W_a^{\vset\setminus\{x\}}}+3$ because $\variance{Q_x^a}=3$ regardless of whether $Q_x^a$ equals $P_x$ or $N_x$.
  By applying this formula recursively for every Boolean variable $x$, we have 
  $\variance{W_a^\vset}=3(1+4+\cdots+4^{n-1})=4^n-1$.
  When $a\neq b$, by letting $x$ be a Boolean variable satisfying $a(x)\neq b(x)$, we have $\covariance{Q^a_x}{Q^b_x}=\covariance{P_x}{N_x}=-1$.
  By substituting $\covariance{Q^a_x}{Q^b_x}$ in (\ref{eq:lemassignment}), $\covariance{W_a^\vset}{W_b^\vset}=-1$.
\end{proof}
\begin{lemma}
  \label{lem:modelcount}
  Let $f,g$ be Boolean functions on variable set $\vset$.
  Then, under identical settings of expectations and (co)variances as Lemma~\ref{lem:assignment}, $|\modelset_f^\vset|=\lceil\variance{W_f^\vset}/(4^n-1)\rceil$ and $|\modelset_{f\wedge g}^\vset|=\lceil\covariance{W_f^\vset}{W_g^\vset}/(4^n-1)\rceil$.
\end{lemma}
\begin{proof}
  By recursively applying (\ref{eq:sumcov}), we have
  \begin{align*}
    \covariance{W_f^\vset}{W_g^\vset} & = \sum_{a\in\modelset_f^\vset}\sum_{b\in\modelset_g^\vset}\covariance{W_a^\vset}{W_b^\vset} \\
    & = \sum_{a\in\modelset_f^\vset\cap\modelset_g^\vset}\variance{W_a^\vset} + \sum_{(a,b)\in\modelset_f^\vset\times\modelset_g^\vset:a\neq b}\covariance{W_a^\vset}{W_b^\vset}.
  \end{align*}
  The first term is $(4^n-1)|\modelset_{f\wedge g}^\vset|$ and the second term is $-|\{(a,b)\in\modelset_f^\vset\times\modelset_g^\vset\mid a\neq b\}|$ by Lemma~\ref{lem:assignment}.
  The latter can be lower bounded by $-|\{(a,b)\in\modelset_\mathtrue^\vset\times\modelset_\mathtrue^\vset\mid a\neq b\}|=-2^n(2^n-1)=-(4^n-2^n)>-(4^n-1)$.
  Thus, we have
  \begin{equation*}
    (4^n\!-\!1)(|\modelset_{f\wedge g}^\vset|\!-\!1)<\covariance{W_f^\vset}{W_g^\vset}\leq (4^n\!-\!1)|\modelset_{f\wedge g}^\vset|,
  \end{equation*}
  indicating $|\modelset_{f\wedge g}^\vset|=\lceil\covariance{W_f^\vset}{W_g^\vset}/(4^n-1)\rceil$.
  By setting $g=f$, we have $|\modelset_f^\vset|=\lceil\variance{W_f^\vset}/(4^n-1)\rceil$.
\end{proof}

Lemma~\ref{lem:modelcount} indicates the reductions from \CTquery to \VCquery and from \SEquery to \CVCquery.
Given Boolean function $f$, we can answer \CTquery by computing $\variance{W_f}$ under the settings of Lemma~\ref{lem:assignment}.
Given Boolean functions $f,g$, we can answer \SEquery as follows.
We compute $\variance{W_f}$ and $\covariance{W_f}{W_g}$ under the settings of Lemma~\ref{lem:assignment} and obtain $|\modelset_f^\vset|$ and $|\modelset_{f\wedge g}^\vset|$ by Lemma~\ref{lem:modelcount}.
Then $f\models g$ if and only if $|\modelset_f^\vset|=|\modelset_{f\wedge g}^\vset|$.
Combined with the intractability results of \CTquery and \SEquery, the intractability of \VCquery for st-DNNFs and that of \CVCquery for d-DNNFs and FBDDs follow.
Note that \CVCquery is also intractable for st-DNNFs since \VCquery can be reduced to \CVCquery with $g=f$.

The remaining is to show the intractability of \VCquery for d-DNNFs and FBDDs.
We use the following lemma.
\begin{lemma}
  \label{lem:ite}
  Let $f,g$ be Boolean functions on variable set $\vset$, $z\notin\vset$ be a Boolean variable, and $h=(z\wedge f)\vee(\neg z\wedge g)$.
  We set $\mu_{P_z}=\mu_{N_z}=1$ and  $\sigma_{P_z}^2=\sigma_{N_z}^2=-\sigma_{P_zN_z}=3$, and $N_x$ and $P_x$ $(x\in\vset)$ have identical settings as Lemma~\ref{lem:assignment}.
  Then, $\covariance{W_f^\vset}{W_g^\vset}=\variance{W_f^\vset}+\variance{W_g^\vset}-\variance{W_h^{\vset\cup\{z\}}}/4+3(\expect{W_f^\vset}-\expect{W_g^\vset})^2/4$.
\end{lemma}
\begin{proof}
  Since $W_h^{\vset\cup\{z\}}\!=\!P_zW_f^\vset+N_zW_g^\vset$, $\variance{W_h^{\vset\cup\{z\}}}\!=\!\variance{P_zW_f^\vset}+\variance{N_zW_g^\vset}+2\covariance{P_zW_f^\vset}{N_zW_g^\vset}$ by (\ref{eq:sumcov}).
  We have $\variance{P_zW_f^\vset}=4\variance{W_f^\vset}+3(\expect{W_f^\vset})^2$, $\variance{N_zW_g^\vset}=4\variance{W_g^\vset}+3\expect{W_g^\vset}^2$, and $\covariance{P_zW_f^\vset}{N_zW_g^\vset}=-2\covariance{W_f^\vset}{W_g^\vset}-3\expect{W_f^\vset}\expect{W_g^\vset}$ by using (\ref{eq:prodcov}).
  Substituting each term leads to the equation in Lemma~\ref{lem:ite}.
\end{proof}
Lemma~\ref{lem:ite} demonstrates that we can obtain $\covariance{W_f}{W_g}$ by computing the variances of $W_f$,$W_g$, and $W_h$ and the expectations of $W_f$ and $W_g$.
If $f,g$ are given as FBDDs, we can easily construct the FBDD of $h$ by simply adding decision node $\vee$ at the root: $(z\wedge f)\vee(\neg z\wedge g)$, which does not break the restrictions of FBDDs.
This construction is also valid for d-DNNFs when $f,g$ are given as d-DNNFs.
Thus, \CVCquery can be answered by solving \VCquery when $f,g$ are given as d-DNNFs or FBDDs; they also admit the expectation computation because it amounts to ordinal WMC.
This indicates the intractability of \VCquery for d-DNNFs and FBDDs, proving Theorem~\ref{thm:intractability}.
Table~\ref{tb:tractability} summarizes the tractability of the queries.

\begin{table}[tb]
  \centering
  {\tabcolsep=3pt
  \begin{tabular}{ccccc}
    \toprule
    Query & st-d-DNNF & st-DNNF & d-DNNF & FBDD \\
    \midrule
    \VCquery & \checkmark (Thm.~\ref{thm:tractable}) & $\circ$ (Thm.~\ref{thm:intractability}) & $\circ$ (Thm.~\ref{thm:intractability}) & $\circ$ (Thm.~\ref{thm:intractability}) \\
    \CVCquery & $\checkmark^\ast$ (Thm.~\ref{thm:tractable}) & $\circ$ (Thm.~\ref{thm:intractability}) & $\circ$ (Thm.~\ref{thm:intractability}) & $\circ$ (Thm.~\ref{thm:intractability}) \\
    \CTquery & \checkmark & $\circ$ & \checkmark & \checkmark \\
    \SEquery & $\checkmark^\ast$ & $\circ$ & $\circ$ & $\circ$ \\
    \bottomrule
    \multicolumn{5}{c}{${}^\ast$Assuming that two st-d-DNNFs respect the same vtree.}
  \end{tabular}
  }
  \caption{Tractability of queries. \checkmark\ indicates that this query can be answered in polynomial time in the sizes of NNFs, and $\circ$ indicates that it cannot be answered in polynomial time unless P=NP.}
  \label{tb:tractability}
\end{table}

\section{Application for Bayesian Networks}
We introduce an application that considers the uncertainty in the inference of Bayesian networks.
A discrete Bayesian network represents a joint distribution over categorical random variables $\bnset\coloneqq\{X_1,\ldots,X_n\}$, where the range of $X_i$ is $\{x_{i1},\ldots,x_{ik_i}\}$.
Each random variable $X_i$ has parents $\parset_i\subseteq\bnset$, and the dependence structure is assumed to be acyclic.
The joint probability that $X_i$ takes value $x_i$ for $i=1,\ldots,n$ is described as $\prob(x_1,\ldots,x_n)=\prod_{i=1}^{n}\prob(x_i\vert\parval_i)$, where $\parval_i\coloneqq\{u_{i1},\ldots,u_{i\ell_i}\}$ is the set of values of parent variables $\parset_i$.
A marginal inference for a Bayesian network is to compute the marginal probability of \emph{partial assignments} that are the values of some random variables.

In an ordinal setting, every conditional distribution is characterized by a set of fixed parameters.
More specifically, distribution $\prob(x_i\vert\parval_i)$ for given parent values $\parval_i$ is a Bernoulli (for a binary-valued $X_i$) or a categorical (for a general $X_i$) distribution with fixed parameters.
Since these parameters are often learned from data, they may have uncertainty.
As explained in the Introduction, a Bayesian statistical approach to model the uncertainty is to introduce distributions, e.g., beta or Dirichlet distributions, for the parameters and regard the marginal probability as a random variable.
With our method, we can compute the variance of the marginal probability.
Our main result here is as follows.

\begin{theorem}
  \label{thm:bayesian}
  Given a Bayesian network with a constant treewidth, we can compute the variance of a marginal probability in polynomial time.
\end{theorem}

This theorem can be proved by using the existing WMC encoding~\citep{chavira08enc} of Bayesian networks and then compute the marginal probability's variance by our proposed algorithm.
Since the method for the general case is complicated, as it requires the slight relaxation of the independence assumption, we defer the details of the general method and the proof of Theorem~\ref{thm:bayesian} to Appendix~\ref{app:multivalued}.
Instead, we here explain a simpler method for the case where every random variable is binary-valued, i.e., $k_i=2$ for every $X_i$.
We use the encoding of \citet{sang05enc2}, referred to as ENC2 by \citet{chavira08enc}.

Before incorporating uncertainty, we explain the method for ordinal marginal inference.
For every random variable $X_i$, we prepare indicator variables $\lambda_{x_{i1}},\lambda_{x_{i2}}$; $\lambda_{x_{ij}}=\mathtrue$ when $X_i=x_{ij}$.
We set the following clauses:
\begin{equation}
  \lambda_{x_{i1}}\vee\lambda_{x_{i2}},\quad\neg\lambda_{x_{i1}}\vee\neg\lambda_{x_{i2}}. \label{eq:indicator}
\end{equation}
We also prepare parameter variable $\rho_{x_{i1}\vert\parval_i}$ for every pattern on parent values $\parval_i$ and set the following clauses:
\begin{equation}
  \begin{aligned}
    \lambda_{u_{i1}}\wedge\cdots\wedge\lambda_{u_{i\ell_i}}\wedge\rho_{x_{i1}\vert\parval_i} & \implies\lambda_{x_{i1}}, \\
    \lambda_{u_{i1}}\wedge\cdots\wedge\lambda_{u_{i\ell_i}}\wedge\neg\rho_{x_{i1}\vert\parval_i} & \implies\lambda_{x_{i2}}.
  \end{aligned} \label{eq:parameter}
\end{equation}
In addition, we set $P_{\rho}=1-N_{\rho}=\prob(x_{i1}\vert\parval_i)$ for every $\rho=\rho_{x_{i1}\vert\parval_i}$.
Let $f$ be a Boolean function that is a conjunction of all the clauses in (\ref{eq:indicator}) and (\ref{eq:parameter}).
Then the marginal probability given partial assignment $\evidence$ can be obtained in two ways:
(i) To prepare Boolean function $g_\evidence$, which is a conjunction of indicator variables corresponding to $\evidence$, and compute the WMC of $f\wedge g_\evidence$ with $P_{\lambda}=N_{\lambda}=1$ for every $\lambda=\lambda_{ij}$.
(ii) To set $P_{\lambda}=0$ for $\lambda=\lambda_{x_{ij}}$ such that $\evidence$ contains $x_{ij^\prime}\ (j^\prime\neq j)$, set all the other weights of the indicator variables to $1$, and then compute the WMC of $f$.
For example, when $\evidence=\{x_{11},x_{32}\}$, method (i) prepares $g_\evidence=\lambda_{x_{11}}\wedge\lambda_{x_{32}}$, while method (ii) sets $P_{\lambda}=0$ for $\lambda=\lambda_{x_{12}},\lambda_{x_{31}}$.

To incorporate uncertainty, we regard $P_x$ and $N_x$ as random variables.
Expectations $\mu_{P_x},\mu_{N_x}$ are set to the original weight values.
Since the weights of indicator variables $\lambda=\lambda_{x_{ij}}$ are determined regardless of the probability values, we set $\sigma_{P_{\lambda}}^2=\sigma_{N_{\lambda}}^2=\sigma_{P_{\lambda}N_{\lambda}}=0$.
For parameter variables $\rho=\rho_{x_{i1}\vert\parval_i}$, we consider variance $\sigma_{x_{i1}\vert\parval_i}^2$ of probability parameter $\prob(x_{i1}\vert\parval_i)$.
Since $P_{\rho}+N_{\rho}=1$, we set  $\sigma_{{P_{\rho}}}^2=\sigma_{N_{\rho}}^2=-\sigma_{P_{\rho}N_{\rho}}=\sigma_{x_{i1}\vert\parval_i}^2$.
By computing the variance of the WMC under this setting, we can compute the variance of the marginal.
Note that we here assume that each parameter is independent of the others because $(P_x,N_x)$ and $(P_y,N_y)$ ($x\neq y$) are independent, which is justified by the widely-adopted \emph{parameter independence} assumption~\citep{spiegelhalter90paraindep} when parameters are learned from data; see also~\citep{heckerman08bntutorial}.
In the following experiments, we empirically validate the tractability of the proposed algorithm and showcase the usage of variance computation with this encoding.

We finally mention that the variance of the \emph{conditional} probability of a Bayesian network can be \emph{approximately} obtained by using the \CVCquery query.
The conditional probability of $\evidence$ given condition $\condit$ equals $W_{h^\prime}/W_{h}$ with $h^\prime=f\wedge g_{\evidence}\wedge g_{\condit}$ and $h=f\wedge g_{\condit}$ using method (i), i.e., to prepare Boolean function $g_\evidence$.
Although we cannot precisely determine $\variance{W_{h^\prime}/W_{h}}$, by Taylor expansion~\citep{vankempen00frac} we have
$\variance{W_{h^\prime}/W_{h}}\approx\variance{W_{h^\prime}}/\{\expect{W_h}\}^2-2\covariance{W_{h^\prime}}{W_h}\expect{W_{h^\prime}}/\{\expect{W_h}\}^3+\variance{W_h}\{\expect{W_{h^\prime}}\}^2/\{\expect{W_h}\}^4$.

\section{Experiments}
In our experiment, we first confirmed the practical tractability of the proposed algorithm for st-d-DNNFs with an application for computing the variance of the marginal of Bayesian networks.
We used Bayesian networks from bnRep~\citep{leonelli25bnrep}, which collects networks from recent academic literature in various areas.
We retrieved all 70 binary Bayesian networks from bnRep.
The number of random variables ranges from 3 to 122.
We derived the CNF of $f$ with ENC2 by Ace v3.0 (\url{http://reasoning.cs.ucla.edu/ace/}) and compiled every CNF into a SDD, which is a subset of an st-d-DNNF, by the SDD package~\citep{choi13sdd}.
Given $p=\prob(x_{i1}\vert\parval_i)$ in the data, we set $\sigma_{x_{i1}\vert\parval_i}^2=p(1-p)/\theta$, which virtually considered $\prob(x_{i1}\vert\parval_i)$ follows $\text{Beta}((\theta-1)p,(\theta-1)(1-p))$.
We set $\theta=10$; note that the value of $\theta$ does not affect the computational time.
For parameters where $p=0$ or $1$, we set $\sigma_{x_{i1}\vert\parval_i}^2=0$ because $\prob(x_{i1}\vert\parval_i)$ should take a value within $[0,1]$, and thus the variance must be $0$ when the expectation is $0$ or $1$.
We chose one random variable from a Bayesian network as a partial assignment and computed the variance of the marginal by method (ii).
Note that the choices of the partial assignment and the expectations and (co)variances of weights do not affect the computational time since the size of the SDD representing $f$ remains unchanged.
The proposed method was implemented in C++ and compiled with g++-11.4.0.
All experiments were performed on a single thread of a Linux server with AMD EPYC 7763 CPU and 2048 GB RAM; note that we used less than 4 GB of memory during the experiments.
We reported the average consumed time for SDD compilation and variance computation over 10 runs for each network.
Codes and data to reproduce the experimental results are available at \url{https://github.com/nttcslab/variance-wmc}.

\begin{table}[tb]
  \centering
  {\footnotesize\tabcolsep=2pt
  \begin{tabular}{lrrrr}
    \toprule
    \multicolumn{1}{c}{Name} & \multicolumn{1}{c}{\#rv} & \multicolumn{1}{c}{$|\text{SDD}|$} & \multicolumn{1}{c}{Compile (s)} & \multicolumn{1}{c}{Variance (s)} \\
    \midrule
    projectmanagement & 26 & 3888 & 0.500 & 0.025 \\
    GDIpathway2 & 28 & 2755 & 0.784 & 0.021 \\
    grounding & 36 & 3397 & 2.387 & 0.017 \\
    engines & 12 & 1804 & 0.240 & 0.011 \\
    windturbine & 122 & 2043 & 1.380 & 0.009 \\
    \bottomrule
  \end{tabular}
  }
  \caption{Top-5 (out of 70) time-consuming networks. ``\#rv'' is the number of random variables. ``$|\text{SDD}|$'' is the size of compiled SDD. ``Compile'' and ``Variance'' indicate the time required to compile SDD and compute variance.}
  \label{tb:experiment0}
\end{table}

As a result, after the SDD was compiled, the variance computation only took 0.025 sec at maximum, recorded for ``projectmanagement'' network whose SDD size was 3,888.
Even if the SDD compilation is added to the computational time, it took only 10 sec to process the most time-consuming ``propellant'' network.
Table~\ref{tb:experiment0} shows the top-5 networks in descending order of the variance computation times.
This indicates the practical tractability of the proposed algorithm.
More detailed results can be found in Appendix~\ref{app:experiment}.

\begin{figure}[t]
  \centering
  \begin{minipage}{0.4\columnwidth}
    \centering
    \includegraphics[keepaspectratio]{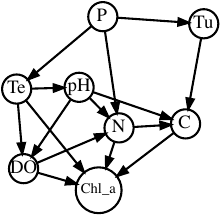}
    \caption{The ``algalactivity2'' network.}
    \label{fig:experimentbn}
  \end{minipage}\noindent
  \hspace{0.02\columnwidth}
  \begin{minipage}{0.55\columnwidth}
      \centering
      {\footnotesize\tabcolsep=2pt
      \begin{tabular}{lr}
        \toprule
        \multicolumn{1}{c}{Parameter} & \multicolumn{1}{c}{Variance} \\
        \midrule
        $\text{DO}\vert\text{pH}_0,\text{Te}_0$ & 0.002887 \\
        $\text{Chl\_a}\vert\text{C}_1,\text{DO}_0,\text{N}_0,\text{Te}_1$ & 0.003532 \\
        $\text{Te}\vert\text{P}_0$ & 0.003554 \\
        $\text{pH}\vert\text{Te}_0$ & 0.003592 \\
        $\text{Chl\_a}\vert\text{C}_1,\text{DO}_1,\text{N}_1,\text{Te}_0$ & 0.003674 \\
        \midrule
        (none) & 0.003904 \\
        \bottomrule
      \end{tabular}
      }
      \captionof{table}{Variance of $\prob(\text{Chl\_a}\!=\!0)$ when one parameter's variance is reduced to one-tenth. Top-5 parameters in ascending order of variance was exhibited.}
      \label{tb:showcase}
  \end{minipage}\noindent
\end{figure}

Next, we showcased the usage of variance computations with the ``algalactivity2'' network from bnRep (Fig.~\ref{fig:experimentbn}).
Each random variable in Fig.~\ref{fig:experimentbn} is valued either $0$ or $1$.
With the same setting as the above experiment, the mean and variance of $\prob(\text{Chl\_a}\!=\!0)$ are computed as 0.5281 and 0.003904.
Since the standard deviation is 0.06248, the variance will affect the decision-making that depends on, e.g., whether $\prob(\text{Chl\_a}\!=\! 0)\leq 0.55$.
To conduct more robust decision-making, we want to reduce the variance of the marginal.
One approach is to decrease the variance of the parameters by collecting more observations (data).
However, since it is costly to collect observations corresponding to all the parameters, we want to find parameters that are effective for reducing the variance of the marginal.
Thus, we additionally demonstrated how much the variance of this marginal is decreased by reducing the variance of one parameter to one-tenth.
We conducted the above demonstration for each of the 43 parameters in the network.
Table~\ref{tb:showcase} shows the top-5 parameters in the reduction of the variance of $\prob(\text{Chl\_a}\!=\! 0)$.
In other words, it shows the top-5 parameters having greater impact on the variance of the marginal.
Here, $\text{RV}_j$ stands for $\text{RV}\!=\!j$; e.g., $\text{DO}\vert\text{pH}_0,\text{Te}_0$ denotes parameter $\prob(\text{DO}\vert\text{pH}\!=\!0,\text{Te}\!=\!0)$.
It is notable that, among the 43 parameters, those having greater impact on the variance of $\prob(\text{Chl\_a}\!=\!0)$ are not only the conditional probabilities of Chl\_a but also those of the other random variables.
We realized that reducing the variance of the parameters in Table~\ref{tb:showcase} efficiently decreased the variance of the marginal.
These results suggest us that we cannot reveal what parameters have greater impact on the variance of the inference result without actually computing the variance.
We give more examples of the variance computation in Appendix~\ref{app:experiment}.

\section{Conclusion}
We defined a query for computing the WMC's variance.
We proved that this query is tractable for st-d-DNNFs and intractable for st-DNNFs, d-DNNFs, and FBDDs; the tractability was shown by presenting an algorithm to solve the query.
We also showed an application for quantifying the uncertainty in the inference on Bayesian networks.
Future directions include the computation of more involved WMC statistics, such as higher-order moments.
We should also investigate the tractability of \VCquery and \CVCquery queries for emerging classes of representations, such as and-sum circuits~\citep{onaka25stdasc}.

\bibliographystyle{abbrvnat}
\bibliography{mybib}

\appendix

\section{Preprocessing and Expectation}
\label{app:preprocessing}
In this appendix, we show the details of the preprocessing and adjust functions ($\adjexpfunc$ and $\adjcovfunc$) as well as a procedure for computing expectations.

We need an adjustment of variable sets when we have the value of expectation $\expect{W_f^{\vset^\prime}}$ or covariance $\covariance{W_f^{\vset^\prime}}{W_g^{\vset^\prime}}$ on variable set $\vset^\prime$ and must compute $\expect{W_f^{\vset^{\prime\prime}}}$ or $\covariance{W_f^{\vset^{\prime\prime}}}{W_g^{\vset^{\prime\prime}}}$, where $\vset^{\prime}\subsetneq\vset^{\prime\prime}$.
Particularly, in Algorithm~\ref{alg:prod}, we need such adjustments where $\vset^\prime=\scope(\vnode)$ and $\vset^{\prime\prime}=\scope(\wnode)$ for vnodes $\vnode,\wnode$, where $\wnode$ is an ancestor of $\vnode$.
The key lemma is as follows.
\begin{lemma}
  \label{lem:preprocessing}
  Given Boolean functions $f,g$ on variable sets $\vset^\prime$ and $\vset^\prime\subseteq\vset^{\prime\prime}$, we have
  \begin{align}
    \expect{W_f^{\vset^{\prime\prime}}} & = \expect{W_{\mathtrue}^{\vset^{\prime\prime}\setminus\vset^\prime}}\expect{W_f^{\vset^\prime}}, \label{eq:preprocessexp} \\
    \covariance{W_f^{\vset^{\prime\prime}}}{W_g^{\vset^{\prime\prime}}} & = \variance{W_\mathtrue^{\vset^{\prime\prime}\setminus\vset^\prime}}\covariance{W_f^{\vset^\prime}}{W_g^{\vset^\prime}} \nonumber \\
    & + \variance{W_\mathtrue^{\vset^{\prime\prime}\setminus\vset^\prime}}\expect{W_f^{\vset^\prime}}\expect{W_g^{\vset^\prime}} \label{eq:preprocesscov} \\
    & + \{\expect{W_\mathtrue^{\vset^{\prime\prime}\setminus\vset^\prime}}\}^2\covariance{W_f^{\vset^\prime}}{W_g^{\vset^\prime}}. \nonumber
  \end{align}
\end{lemma}
\begin{proof}
  On variable set $\vset^{\prime\prime}$, we can represent $f$ as $\mathtrue^{\vset^{\prime\prime}\setminus\vset^\prime}\wedge f$, where $\mathtrue^{\vset^{\prime\prime}\setminus\vset^\prime}$ is a $\mathtrue$ Boolean function on variable set $\vset^{\prime\prime}\setminus\vset^\prime$.
  We can also represent $g$ on $\vset^{\prime\prime}$ as $\mathtrue^{\vset^{\prime\prime}\setminus\vset^\prime}\wedge g$.
  Thus, we have $W_f^{\vset^{\prime\prime}}=W_{\mathtrue}^{\vset^{\prime\prime}\setminus\vset^\prime}W_f^{\vset^\prime}$ and $W_g^{\vset^{\prime\prime}}=W_{\mathtrue}^{\vset^{\prime\prime}\setminus\vset^\prime}W_f^{\vset^\prime}$, where $W_{\mathtrue}^{\vset^{\prime\prime}\setminus\vset^\prime}$ and $(W_f^{\vset^\prime},W_g^{\vset^\prime})$ are independent.
  Eq.~(\ref{eq:preprocessexp}) follows from the formula for the expectation of independent random variables and (\ref{eq:preprocesscov}) follows from (\ref{eq:prodcov}).
\end{proof}

\begin{algorithm}[tb]
\caption{Preprocessing and adjustment functions}
\label{alg:preprocess}
{\footnotesize
\ForEach{$\vnode$: vnode of $\vtree$ in a bottom-up order}{
  \uIf(\tcp*[f]{$x$: label of $\vnode$}){$\vnode$ is a leaf node}{
    $\vexpvar[\vnode]\leftarrow\mu_{P_x}+\mu_{N_x}$ \tcp*{$\vexpvar[\vnode]=\expect{W_{\mathtrue}^{\scope(\vnode)}}$}
    $\vvarvar[\vnode]\leftarrow\sigma^2_{P_x}+\sigma^2_{N_x}+2\sigma_{P_x N_x}$ \tcp*{$\vvarvar[\vnode]=\variance{W_{\mathtrue}^{\scope(\vnode)}}$}
  }
  \Else{
    $\vexpvar[\vnode]\leftarrow\vexpvar[\vnode^l]\cdot\vexpvar[\vnode^r]$\;
    $\vvarvar[\vnode]\leftarrow\vvarvar[\vnode^l]\!\cdot\!\vvarvar[\vnode^r]\!+\!\vvarvar[\vnode^l]\!\cdot\!(\vexpvar[\vnode^r])^2\!+\!(\vexpvar[\vnode^l])^2\!\cdot\!\vvarvar[\vnode^r]$\;
  }
}
\ForEach{$\vnode$: vnode of $\vtree$}{
  $\adjexpvar[\vnode,\vnode]\leftarrow 1$, $\adjvarvar[\vnode,\vnode]\leftarrow 0$ \tcp*{$\adjexpvar[\wnode,\vnode]=\expect{W_{\mathtrue}^{\scope(\wnode)\setminus\scope(\vnode)}}$}
  $\wnode\leftarrow\vnode$ \tcp*{$\adjvarvar[\wnode,\vnode]=\variance{W_{\mathtrue}^{\scope(\wnode)\setminus\scope(\vnode)}}$}
  \While{$\wnode$ is not a root node of $\vtree$}{
    $\pnode\leftarrow\text{parent of $\wnode$}$\;
    \lIf{$\pnode^l=\wnode$}{$\mathtt{etmp}\leftarrow\vexpvar[\pnode^r]$, $\mathtt{vtmp}\leftarrow\vvarvar[\pnode^r]$}
    \lElse{$\mathtt{etmp}\leftarrow\vexpvar[\pnode^l]$, $\mathtt{vtmp}\leftarrow\vvarvar[\pnode^l]$}
    $\adjexpvar[\pnode,\vnode]\leftarrow\adjexpvar[\wnode,\vnode]\cdot\mathtt{etmp}$\;
    $\adjvarvar[\pnode,\vnode]\leftarrow\adjvarvar[\wnode,\vnode]\!\cdot\!\mathtt{vtmp}\!+\!\adjvarvar[\wnode,\vnode]\!\cdot\!(\mathtt{etmp})^2\!+\!(\adjexpvar[\wnode,\vnode])^2\!\cdot\!\mathtt{vtmp}$\;
    $\wnode\leftarrow\pnode$\;
  }
}
\Input{Vnodes $\wnode,\vnode$, $\mathtt{val}=\expect{W_f^{\scope(\vnode)}}$ for some function $f$}
\Output{$\expect{W_f^{\scope(\wnode)}}$\tcp*{Suppose $\scope(\vnode)\subseteq\scope(\wnode)$}}
\Func{$\adjexpfunc(\wnode,(\vnode,\mathtt{val}))$}{
  \lIf(\tcp*[f]{$f\in\{\mathtrue,\mathfalse\}$}){$\vnode=\bot$}{\Return $\vexpvar[\wnode]\cdot\mathtt{val}$}
  \lElse{\Return $\adjexpvar[\wnode,\vnode]\cdot\mathtt{val}$}
}
\Input{Vnodes $\wnode,\vnode,\mathsf{vf},\mathsf{vg}$, $\mathtt{val}=\covariance{W_f^{\scope(\vnode)}}{W_g^{\scope(\vnode)}}$, $\mathtt{fexp}=\expect{W_f^{\scope(\mathsf{vf})}}$, $\mathtt{gexp}=\expect{W_g^{\scope(\mathsf{vg})}}$ for some $f,g$}
\Output{$\covariance{W_f^{\scope(\wnode)}}{W_g^{\scope(\wnode)}}$ \tcp*{Suppose $\scope(\vnode)\subseteq\scope(\wnode)$}}
\Func{$\adjcovfunc(\wnode,(\vnode,\mathtt{val}),(\mathsf{vf},\mathtt{fexp}),(\mathsf{vg},\mathtt{gexp}))$}{
  \lIf(\tcp*[f]{$f,g\!\in\!\{\mathtrue,\mathfalse\}$}){$\vnode=\bot$}{\Return $\vvarvar[\wnode]\!\cdot\!\mathtt{fexp}\!\cdot\!\mathtt{gexp}$}
  $\mathtt{atmp}\!\leftarrow\!\adjexpfunc(\vnode,(\mathsf{vf},\mathtt{fexp}))$ \tcp*{Suppose $\scope(\mathsf{vf})\!\subseteq\!\scope(\vnode)$}
  $\mathtt{btmp}\!\leftarrow\!\adjexpfunc(\vnode,(\mathsf{vg},\mathtt{gexp}))$ \tcp*{Suppose $\scope(\mathsf{vg})\!\subseteq\!\scope(\vnode)$}
  \Return $\adjvarvar[\wnode,\vnode]\!\cdot\!\mathtt{val}\!+\!\adjvarvar[\wnode,\vnode]\!\cdot\!\mathtt{atmp}\!\cdot\!\mathtt{btmp}\!+\!(\adjexpvar[\wnode,\vnode])^2\!\cdot\!\mathtt{val}$\;
}
}
\end{algorithm}

Thus, we can compute the adjustment in constant time by pre-computing $\expect{W_\mathtrue^{\scope(\wnode)\setminus\scope(\vnode)}}$ and $\variance{W_\mathtrue^{\scope(\wnode)\setminus\scope(\vnode)}}$ for every $\vnode$ and every ancestor $\wnode$ of $\vnode$.
Algorithm~\ref{alg:preprocess} (lines 1--17) pre-computes the expectation and variance of $W_\mathtrue^{\scope(\wnode)\setminus\scope(\vnode)}$ for every $\vnode,\wnode$.
Lines 1--7 compute the expectation and variance of $W_\mathtrue^{\scope(\vnode)}$ for every $\vnode$ and store the computed values in $\vexpvar[\vnode]$ and $\vvarvar[\vnode]$.
We can easily derive these values when $\vnode$ is a leaf.
When $\vnode$ is an internal vnode, we can use Lemma~\ref{lem:preprocessing} because $\mathtrue^{\scope(\vnode)}=\mathtrue^{\scope(\vnode^l)}\wedge\mathtrue^{\scope(\vnode^r)}$.
Thus, we process the vnodes in the order where the deeper vnode comes earlier.
Lines 8--17 compute the expectation and variance of $W_\mathtrue^{\scope(\wnode)\setminus\scope(\vnode)}$ for every $\vnode,\wnode$ and store the computed values in $\adjexpvar[\wnode,\vnode]$ and $\adjvarvar[\wnode,\vnode]$.
For every vnode $\vnode$, we start with $\wnode=\vnode$ and traverse the ancestors one by one until $\wnode$ reaches the root vnode.
We again use Lemma~\ref{lem:preprocessing} to compute $\adjexpvar[\wnode,\vnode]$ and $\adjvarvar[\wnode,\vnode]$ one by one.

Using $\vexpvar,\vvarvar,\adjexpvar$, and $\adjvarvar$, we can adjust the expectation and covariance by $\adjexpfunc$ in lines 18--20 and $\adjcovfunc$ in lines 21--25, which again uses Lemma~\ref{lem:preprocessing}.
In $\adjcovfunc$, we also need the expectations of $W_f$ and $W_g$ because (\ref{eq:preprocesscov}) requires $\expect{W_f}$ and $\expect{W_g}$.

The preprocessing requires $\order{|\vset|^2}$ time; in lines 8--17, traversing the ancestors for every $\vnode$ takes $\order{|\vset|}$ time, which in total requires $\order{|\vset|^2}$ time.
After preprocessing, $\adjexpfunc$ and $\adjcovfunc$ can be computed in constant time.

\begin{algorithm}[tb]
\caption{$\expfunc(\alpha)$: computing $\expect{W_{f_\alpha}}$ for every node}
\label{alg:expect}
{\footnotesize
\Input{St-d-DNNF $\alpha$}
\Output{Pair of $\vnode=\dnode(\alpha)$ and $\expect{W_{f_\alpha}^{\scope(\vnode)}}$}
\lIf(\tcp*[f]{Cache for $\expfunc(\alpha)$}){$\expvar[\alpha]\neq \textrm{null}$}{\Return $\expvar[\alpha]$}
$\vnode\leftarrow\dnode(\alpha)$\;
\lIf(\tcp*[f]{Base cases: $\alpha$ is a leaf node}){$\alpha=\mathfalse$}{$\resvar\leftarrow 0$}
\lElseIf {$\alpha=\mathtrue$}{$\resvar\leftarrow 1$}
\lElseIf {$\alpha=x$}{$\resvar\leftarrow \mu_{P_x}$}
\lElseIf {$\alpha=\neg x$}{$\resvar\leftarrow \mu_{N_x}$}
\uElseIf(\tcp*[f]{$\alpha_1,\ldots,\alpha_k$: child nodes of $\alpha$}){$\alpha$ is a $\vee$-node}{
  $\resvar\leftarrow\sum_{j=1}^{k}\adjexpfunc(\vnode,\expfunc(\alpha_i))$\;
}
\Else(\tcp*[f]{$\alpha$ is a $\wedge$-node}){
  $\alpha^l,\alpha^r\leftarrow(\text{child nodes of $\alpha$})$ s.t. $\scope(\alpha^l)\subseteq\scope(\vnode^l)$ and $\scope(\alpha^r)\subseteq\scope(\vnode^r)$\;
  $\resvar\leftarrow\adjexpfunc(\vnode^l,\expfunc(\alpha^l))\cdot\adjexpfunc(\vnode^r,\expfunc(\alpha^r))$\;
}
$\expvar[\alpha]\leftarrow (\vnode,\resvar)$\;
\Return $\expvar[\alpha]$
}
\end{algorithm}

Using $\adjexpfunc$, we can also compute expectation $\expect{W_{f_\gamma}}$ for every node $\gamma$ in the same way as the standard algorithm for computing an ordinal WMC.
Algorithm~\ref{alg:expect} describes the procedure.
When $\alpha$ is a $\vee$-node, we can compute the expectation of $f_\alpha$ by summing up the expectations of the child nodes due to determinism, i.e., $\expect{W_{f_\alpha}^\vset}=\sum_{i=1}^{k}\expect{W_{f_{\alpha_i}}^\vset}$, which is reflected in lines 7 and 8.
Here, to adjust the variable sets, we use the $\adjexpfunc$ function.
When $\alpha$ is a $\wedge$-node, we can compute the expectation of $f_\alpha$ as the product of the expectations of both child nodes, as reflected in lines 9--11.
For st-d-DNNF $\alpha$, computing the expectations for all the descendant nodes of $\alpha$ takes $\order{|\alpha|}$ time after preprocessing.

\section{Bayesian Networks with Multi-Valued Random Variables}
\label{app:multivalued}
In this appendix, we describe how to compute the variance of the marginal probability of a Bayesian network when not every variable is binary.
In this case, we use the encoding of~\citet{darwiche02enc1}, denoted as ENC1 by~\citet{chavira08enc}.
We also prove Theorem~\ref{thm:bayesian} stated in the main text for general cases using this encoding.

\subsection{Encoding and Algorithm}
We first describe the method for ordinal inference using ENC1.
The indicator variables are prepared in the same way: for every random variable $X_i$, we prepare $\lambda_{x_{i1}},\ldots,\lambda_{x_{ik_i}}$; $\lambda_{x_{ij}}=\mathtrue$ when $X_i=x_{ij}$.
We set the following clauses,
\begin{equation}
  \begin{aligned}
    & \lambda_{x_{i1}}\vee\cdots\vee\lambda_{x_{ik_i}}, \\
    & \neg\lambda_{x_{ij}}\vee\neg\lambda_{x_{ij^\prime}}\quad (j\neq j^\prime),\\
  \end{aligned}\label{eq:indicator2}
\end{equation}
ensuring that, among $\lambda_{x_{i1}},\ldots,\lambda_{x_{ik_i}}$, exactly one variable is set to $\mathtrue$.
The parameter variables are prepared differently.
We prepare parameter variable $\theta_{x_{ij}\vert\parval_i}$ for every $j\in\{1,\ldots,k_i\}$ and every pattern on parent values $\parval_i$, and set the following clause:
\begin{equation}
  \lambda_{u_{i1}}\wedge\cdots\lambda_{u_{i\ell_i}}\wedge\lambda_{x_{ij}} \iff \theta_{x_{ij}\vert\parval_i}. \label{eq:parameter2}
\end{equation}
Let $f$ be a Boolean function that is a conjunction of all the clauses in (\ref{eq:indicator2}) and (\ref{eq:parameter2}).
For parameter variables, we set $N_\theta=1$ and $P_\theta=\prob(x_{ij}\vert\parval_i)$ for $\theta=\theta_{x_{ij}\vert\parval_i}$.
Then, the marginal probability of $\evidence$ can be obtained in the same way as the binary-valued case:
either (i) to prepare a Boolean function $g_\evidence$ that is a conjunction of indicator variables corresponding to $\evidence$ and compute the WMC of $f\wedge g_\evidence$ with $P_{\lambda}=N_{\lambda}=1$ for every $\lambda=\lambda_{ij}$, or (ii) to set $P_{\lambda}=0$ for $\lambda=\lambda_{x_{ij}}$ such that $\evidence$ contains $x_{ij^\prime}\ (j^\prime\neq j)$, set all other weights of indicator variables to $1$, and then compute the WMC of $f$.

To obtain the variance of the marginal probability, we must cope with the correlation between the positive weights of \emph{different} parameter variables even under the parameter independence assumption~\citep{spiegelhalter90paraindep}.
For $\theta=\theta_{x_{ij}\vert\parval_i}$ and $\theta^\prime=\theta_{x_{ij^\prime}\vert\parval_i}$ $(j\neq j^\prime)$, $P_\theta=\prob(x_{ij}\vert\parval_i)$ and $P_{\theta^\prime}=\prob(x_{ij^\prime}\vert\parval_i)$.
In general, these probabilities should be correlated since $\sum_{j=1}^{k_i}\prob(x_{ij}\vert\parval_i)=1$.
For example, if $\prob(x_i\vert\parval_i)$ follows a Dirichlet distribution, $P_\theta$ and $P_{\theta^\prime}$ should be correlated and thus have a non-zero covariance.
More specifically, under this situation we set $\mu_{P_x}$ and $\mu_{N_x}$ to the original parameter values for every Boolean variable, $\sigma_{P_\lambda}^2=\sigma_{N_\lambda}^2=\sigma_{P_\lambda N_\lambda}=0$ for $\lambda=\lambda_{x_{ij}}$, and $\sigma_{P_\theta}^2=\variance{\prob(x_{ij}\vert\parval_i)}$ and  $\sigma_{N_\theta}^2=\sigma_{P_\theta N_\theta}=0$ (since $N_\theta$ is an exact value) for $\theta=\theta_{x_{ij}\vert\parval_i}$.
In addition, we also have $\covariance{P_\theta}{P_{\theta^\prime}}=\covariance{\prob(x_{ij}\vert\parval_i)}{\prob(x_{ij^\prime}\vert\parval_i)}$ for $\theta=\theta_{x_{ij}\vert\parval_i}$ and $\theta^\prime=\theta_{x_{ij^\prime}\vert\parval_i}$ $(j\neq j^\prime)$.
Therefore, we cannot straightforwardly apply the proposed algorithm.

We address this situation by restricting the vtree shape and considering the characteristics of the Boolean function $f$.
Later, we show that the restriction on the shape of a vtree does not incur any theoretical burdens.
We consider the following restriction.
\begin{condition}
  \label{cond:vtree}
  For any $i$ and $\parval_i$, vtree node $\vnode_{x_i\vert\parval_i}$ exists such that $\scope(\vnode_{x_i\vert\parval_i})=\{\theta_{x_{i1}\vert\parval_{i}},\ldots,$ $\theta_{x_{ik_i}\vert\parval_{i}}\}$.
  In other words, all the parameter variables having the same parent value $\parval_i$ are gathered in a subtree rooted at $\vnode_{x_i\vert\parval_i}$.
\end{condition}
We also assume that the whole st-d-DNNF has no $\mathfalse$ leaf unless we represent a $\mathfalse$ Boolean function.
We can ensure this by recursively replacing an internal node that has a $\mathfalse$ child node.
If $\wedge$-node $\alpha$ has $\mathfalse$ as a child node, we simply replace $\alpha$ with $\mathfalse$.
If $\vee$-node $\alpha$ has $\mathfalse$ as a child node, we simply remove this child.
If the $\vee$-node has no child node after the above removal, we replace it with $\mathfalse$.
Eventually, no $\mathfalse$ leaves remain, or the root becomes $\mathfalse$ when the represented Boolean function is $\mathfalse$.

Now we make the following observation.
Let $f_{x_{ij}\vert\parval_i}\coloneqq \theta_{x_{ij}\vert\parval_i}\wedge(\bigwedge_{j^\prime\neq j}\neg\theta_{x_{ij^\prime}\vert\parval_i})$ and $f_{x_{i0}\vert\parval_i}\coloneqq\bigwedge_{j^\prime}\neg\theta_{x_{ij^\prime}\vert\parval_i}$.
In other words, $f_{x_{ij}\vert\parval_i}$ corresponds to an assignment where $\theta_{x_{ij}\vert\parval_i}$ is $\mathtrue$ and all the other $\theta_{x_{ij^\prime}\vert\parval_i}$s are $\mathfalse$, and $f_{x_{i0}\vert\parval_i}$ corresponds to an assignment where all $\theta_{x_{ij}\vert\parval_i}$s are $\mathfalse$.
\begin{observation}
  \label{obs:singleton}
  Consider the st-d-DNNF of $f$ respecting a vtree that satisfies Condition~\ref{cond:vtree}, where $f$ is the conjunction of all clauses in (\ref{eq:indicator2}) and (\ref{eq:parameter2}).
  Then, for any node $\alpha$ with $\dnode(\alpha)=\vnode_{x_i\vert\parval_i}$, $f_\alpha=f_{x_{ij}\vert\parval_i}$ for some $j\in\{0,\ldots,k_i\}$.
\end{observation}
\begin{proof}
  Since $\alpha$ is a node in the st-d-DNNF representing $f$, there exists Boolean function $f^\prime$ on variable set $\vset\setminus\scope(\vnode_{x_{ij}\vert\parval_i})$ such that $f^\prime\wedge f_\alpha\models f$, where $f^\prime$ is a conjunction of the represented Boolean functions of other nodes; this is because the Boolean function represented by an NNF is defined as the conjunctions and disjunctions of the functions represented by the nodes.
  Here, $f^\prime\neq\mathfalse$ because the st-d-DNNF has no $\mathfalse$ leaf.
  Let $a^\prime$ be a model of $f^\prime$ on variables $\vset\setminus\scope(\vnode_{x_{ij}\vert\parval_i})$.
  From Eq.~(\ref{eq:indicator2}), exactly one variable among $\lambda_{x_{i1}},\ldots,\lambda_{x_{ik_i}}$ is set to $\mathtrue$ under $a^\prime$.
  If there exists $u_{ij}\in\parval_i$ such that $a^\prime(\lambda_{u_{ij}})=\mathfalse$, all $\theta_{x_{ij}\vert\parval_i}$ must be $\mathfalse$ due to (\ref{eq:parameter2}), meaning that $f_\alpha=f_{x_{i0}\vert\parval_i}$.
  Otherwise, let $\lambda_{x_{ij}}$ be the only variable where $a^\prime(\lambda_{x_{ij}})=\mathtrue$ among $\lambda_{x_{i1}},\ldots,\lambda_{x_{ik_i}}$.
  Then again due to (\ref{eq:parameter2}), $\theta_{x_{ij}\vert\parval_i}$ must be $\mathtrue$, and the other $\theta_{x_{ij^\prime}\vert\parval_i}$s must be $\mathfalse$, meaning that $f_\alpha=f_{x_{ij}\vert\parval_i}$.
\end{proof}

Here, by letting $f_j\coloneqq f_{x_{ij}\vert\parval_i}$, we have $W_{f_0}=1$ and $W_{f_j}=P_{\theta_{x_{ij}\vert\parval_i}}$, leading to
\begin{align}
  & \covariance{W_{f_j}}{W_{f_{j^\prime}}} \nonumber \\
  & =
  \begin{cases}
    0 & (j=0\text{ or }j^\prime=0) \\
    \variance{\prob(x_{ij}\vert\parval_i)} & (j=j^\prime\neq 0) \\
    \covariance{\prob(x_{ij}\vert\parval_i)}{\prob(x_{ij^\prime}\vert\parval_i)} & (\text{otherwise})
  \end{cases}.\raisetag{14pt}\label{eq:multicov}
\end{align}

\begin{figure*}[!t]
    \centering
    \includegraphics[keepaspectratio,scale=0.95]{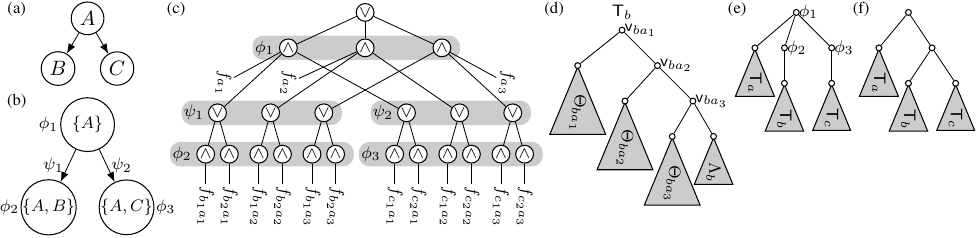}
    \caption{(a) A Bayesian network. (b) A jointree for (a). (c) A d-DNNF built by the algorithm of~\citet{darwiche03ac} with jointree of (b). (d) A vtree $\vtree_b$. (e) Whole vtree before enforcing that every internal vnode has exactly two child vnodes. (f) Whole vtree after enforcing that every internal vnode has exactly two child vnodes.}
    \label{fig:variableelim}
\end{figure*}

Using Observation~\ref{obs:singleton} and (\ref{eq:multicov}), we can modify Algorithm~\ref{alg:prod} to compute the covariance of the WMCs under this situation.
Between lines~10 and 11, we check whether $\lcanode=\vnode_{x_i\vert\parval_i}$ for some $i$ and $\parval_i$.
If so, we have $\dnode(\alpha)=\dnode(\beta)=\vnode_{x_i\vert\parval_i}$, and thus we compute $j,j^\prime\in\{0,\ldots,i_k\}$ such that $f_\alpha=f_{x_{ij}\vert\parval_i}$ and $f_\beta=f_{x_{ij^\prime}\vert\parval_i}$.
Note that this step can be performed in linear time in the sizes of $\alpha,\beta$ by scanning the descendants of $\alpha,\beta$ and investigating which $\theta_{x_{ij}\vert\parval_i}$ is set to $\mathtrue$.
Then, we assign the covariance value computed by (\ref{eq:multicov}) to $\resvar$.
Thus, the time complexity bound in Theorem~\ref{thm:tractable} also holds under this situation.

We must mention here that the adjustments of variable sets in Algorithm~\ref{alg:preprocess} also work accurately under this situation.
Since now $P_{\theta}$ and $P_{\theta^\prime}$ have a non-zero covariance, Algorithm~\ref{alg:preprocess} cannot correctly compute $\variance{W_{\mathtrue}^{\vset^\prime}}$ for $\vset^\prime\subseteq\vset$ that contains some parameter variables.
In other words, $\vvarvar[\vnode]$ and $\adjvarvar[\wnode,\vnode]$ do not correctly store the value of $\variance{W_{\mathtrue}^{\scope(\vnode)}}$ and $\variance{W_{\mathtrue}^{\scope(\wnode)\setminus\scope(\vnode)}}$ when the vertex sets $\scope(\vnode)$ and $\scope(\wnode)\setminus\scope(\vnode)$ contain parameter variables.
However, adjustments using such values do not occur, as we shown below.
An adjustment of variables from $\scope(\vnode)$ to $\scope(\wnode)$ occurs when child node $\alpha$ of an internal node, or a node $\alpha$ itself when $\lcanode$ is neither $\dnode(\alpha)$ nor $\dnode(\beta)$ (lines~5--10 in Algorithm~\ref{alg:prod}), represents the Boolean function $f_\alpha\wedge\mathtrue^{\vset^\prime}$, where $\vset^\prime\coloneqq\scope(\wnode)\setminus\scope(\vnode)$.
With the same argument as the proof of Observation~\ref{obs:singleton}, Boolean function $f^\prime\neq\mathfalse$ exists on variables $\vset\setminus(\vset^\prime\cup\scope(\dnode(\alpha)))$ such that $f^\prime\wedge f_\alpha\wedge\mathtrue^{\vset^\prime}\models f$.
This means that, under any models of $f^\prime\wedge f_\alpha$, the values of the variables in $\vset^\prime$ can be set arbitrarily.
In contrast, according to Observation~\ref{obs:singleton}, every parameter variable must be set to either $\mathtrue$ or $\mathfalse$, depending on the assignment of the indicator variables.
Thus, $\vset^\prime$ never contains parameter variables, which concludes the correctness of the adjustments when executing Algorithm~\ref{alg:prod} under this situation.

\subsection{Proof of Complexity}
Now we show that, even under Condition~\ref{cond:vtree}, we can have a polynomial-sized st-d-DNNF representing $f$ built in polynomial time for a Bayesian network with a constant treewidth.
\begin{proposition}
  \label{prop:bayesian}
  Given a Bayesian network with a constant treewidth, there exists a polynomial-sized st-d-DNNF of $f$ respecting the vtree that satisfies Condition~\ref{cond:vtree}.
  This st-d-DNNF can be built in polynomial time.
\end{proposition}
\citet{darwiche03ac} provided an algorithm to encode $f$ as an arithmetic circuit whose size is polynomially bounded for a Bayesian network with a constant treewidth.
Indeed, this arithmetic circuit can be easily converted to a d-DNNF that represents $f$ by adding negative literals that are not mentioned and replacing addition nodes with $\vee$ and multiplication nodes with $\wedge$.
We show that this converted d-DNNF respects a vtree satisfying Condition~\ref{cond:vtree}.

\begin{proof}
For the sake of completeness, we describe the algorithm of \citet{darwiche03ac} as a d-DNNF compilation algorithm.
For a Bayesian network on variables $\bnset\coloneqq\{X_1,\ldots,X_n\}$, a \emph{jointree} is a rooted node-labeled tree where each label is a subset of $\bnset$.
To distinguish from NNF nodes and vnodes, we respectively call the nodes and edges of jointrees \emph{jnodes} and \emph{jedges}.
The label of jnode $\phi$ is denoted by $\labelset(\phi)$, and label $\labelset(\psi)$ of jedge $\psi$ is defined as the intersection of the labels of the endpoints.
The labeling must satisfy the following conditions: (i) for every $X_i$, $\{X_i\}\cup\parset_i$ must appear in some label, and (ii) for every $X_i$, the jnodes containing $X_i$ as a label are connected.
For example, Fig.~\ref{fig:variableelim}(b) is a jointree for the Bayesian network of Fig.~\ref{fig:variableelim}(a).

Given a jointree, we can construct the d-DNNF of $f$ as follows.
We prepare the following nodes: root $\vee$-node $\alpha$, $\wedge$-node $\beta(\phi,\labelval_\phi)$ for every jnode $\phi$ and pattern (set) $\labelval_\phi$ on the values of $\labelset(\phi)$, and $\vee$-node $\alpha(\psi,\labelval_\psi)$ for every jedge $\psi$ and pattern (set) $\labelval_\psi$ on the values of $\labelset(\psi)$.
The child nodes of root node $\alpha$ are all $\beta(\phi_r,\labelval_{\phi_r})$s, where $\phi_r$ is the root jnode.
The child nodes of $\beta(\phi,\labelval_\phi)$ are $\alpha(\psi,\labelval_\psi)$, satisfying that $\psi$ is a jedge connecting $\phi$ and its child jnode and $\labelval_\psi\subseteq\labelval_\phi$.
The child nodes of $\alpha(\psi,\labelval_\psi)$ are $\beta(\phi,\labelval_\phi)$, satisfying that $\phi$ is the descendant endpoint of $\psi$ and $\labelval_\phi\supseteq\labelval_\psi$.
Moreover, for every $X_i\in\bnset$, we choose jnode $\phi_i$, satisfying $\labelset(\phi_i)=\{X_i\}\cup\parset_i$, and every $\alpha(\phi_i,\labelval_{\phi_i})$ with $\labelval_{\phi_i}=\{x_{ij}\}\cup\parval_i$ also has a child node representing Boolean function $f_{x_{ij}\parval_i}$ defined as:
\begin{align*}
  f_{x_{ij}\parval_i}\!\coloneqq\!f_{x_{ij}\vert\parval_i}\!\wedge\!\left(\bigwedge_{\parval_i^\prime\neq\parval_i}f_{x_{i0}\vert\parval_i^\prime}\right)\!\wedge\!\lambda_{x_{ij}}\!\wedge\!\left(\bigwedge_{j^\prime\neq j}\neg\lambda_{x_{ij^\prime}}\right).
\end{align*}
In other words, by letting $\Theta_{x_i\parval_i}\coloneqq\{\theta_{x_{ij}\vert\parval_i}\ \vert\ j\}$ and $\Lambda_{x_i}\coloneqq\{\lambda_{x_{ij}}\ \vert\ j\}$, $f_{x_{ij}\parval_i}$ is a Boolean function on variables $(\bigcup_{\parval_i}\Theta_{x_i\parval_i})\cup\Lambda_i$, where $\theta_{x_{ij}\vert\parval_i}$ and $\lambda_{x_{ij}}$ are set to $\mathtrue$ and all the other variables are set to $\mathfalse$.
The result of \citet{darwiche03ac} shows that, with an appropriately designed jointree, the compiled d-DNNF represents $f$ and the size is polynomial when the treewidth of the input Bayesian network is constant.
Fig.~\ref{fig:variableelim}(c) depicts the d-DNNF of $f$ constructed by the above algorithm when the jointree of Fig.~\ref{fig:variableelim}(b) is given.
In this example, $A$ takes three values, $a_1,a_2,a_3$, and $B,C$ take two, $b_1,b_2$ and $c_1,c_2$.
Here, for example, $f_{a_1}=\lambda_{a_1}\wedge\neg\lambda_{a_2}\wedge\neg\lambda_{a_3}$ and $f_{b_2a_1}=f_{b_2\vert a_1}\wedge f_{b_0\vert a_2}\wedge f_{b_0\vert a_3}\wedge\lambda_{b_2}\wedge\neg\lambda_{b_1}$, where $f_{b_2\vert a_1}=\theta_{b_2\vert a_1}\wedge\neg\theta_{b_1\vert a_1}$ and $f_{b_0\vert a_i}=\neg\theta_{b_1\vert a_i}\wedge\neg\theta_{b_2\vert a_i}$.
It is clear from the above construction that this d-DNNF can be built in linear time in the size of it.
Since it is shown by~\citep{darwiche03ac} that the size of this d-DNNF is polynomial for a Bayesian network with a constant treewidth, its construction takes polynomial time.

Now we construct a vtree that is respcted by the compiled d-DNNF.
For $X_i\in\bnset$, since every $f_{x_{ij}\parval_i}$ is simply a conjunction of literals, it respects the following vtree $\vtree_{x_i}$: all the patterns (sets) on the parent values are ordered arbitrarily as $\parval_i^1,\ldots,\parval_i^N$, and we prepare internal nodes $\vnode_{x_i\parval_i^1},\ldots,\vnode_{x_i\parval_i^N}$ such that (i) the subtree rooted at $\vnode_{x_i\parval_i^j}$'s left child contains all the variables in $\Theta_{x_i\parval_i^j}$ and (ii) $\vnode_{x_i\parval_i^j}$'s right child is $\vnode_{x_i\parval_i^{j+1}}$, except that the subtree rooted at $\vnode_{x_i\parval_i^N}$'s right child contains all the variables in $\Lambda_{x_i}$.
This vtree $\vtree_{x_i}$, rooted at $\vnode_{x_i\parval_i^1}$, satisfies Condition~\ref{cond:vtree} for this $i$; we can take $\vnode_{x_i\vert\parval_i}$ in Condition~\ref{cond:vtree} as the left child of $\vnode_{x_i\parval_i}$.
Fig.~\ref{fig:variableelim}(d) shows vtree $\vtree_b$ for random variable $B$.
The triangle labeled $\Theta_{ba_i}$ is a sub-vtree containing $\Theta_{ba_i}=\{\theta_{b_1\vert a_i},\theta_{b_2\vert a_i}\}$ as the leaf labels.
In the same manner, the triangle labeled $\Lambda_b$ contains $\Lambda_b=\{\lambda_{b_1},\lambda_{b_2}\}$ as leaf labels.

A whole vtree can be built by attaching each $\vtree_{x_i}$ to the jointree.
Specifically, we attach the root vnode of $\vtree_{x_i}$ as a child node of jnode $\phi_i$ chosen by the above algorithm.
From the description of the compile algorithm, every $\wedge$-node in the compiled d-DNNF clearly decomposes variables according to the built vtree.
By enforcing that every vnode has exactly two child vnodes by chaining the vnodes and also enforcing that every $\wedge$-node has exactly two child nodes in the same manner, we now have a polynomial-sized structured d-DNNF and a vtree satisfying Condition~\ref{cond:vtree}.
Note that enforcing two child nodes only doubles the size.
Fig.~\ref{fig:variableelim}(e) shows the jointree after attaching $\vtree_a,\vtree_b,\vtree_c$.
The decomposition in the d-DNNF of Fig.~\ref{fig:variableelim}(c) clearly respects the (v)tree in Fig.~\ref{fig:variableelim}(e).
To enforce that every $\wedge$-node has two child nodes, we double the nodes corresponding to $\phi_1$ and eliminate those corresponding to $\phi_2$ and $\phi_3$.
The vtree in Fig.~\ref{fig:variableelim}(e) is also transformed in the same way, resulting in the vtree in Fig.~\ref{fig:variableelim}(f).
Then, we obtain the final structured d-DNNF respecting the vtree that satisfies Condition~\ref{cond:vtree}.
\end{proof}

This constitutes the proof for Theorem~\ref{thm:bayesian}.
\begin{proof}[Proof of Theorem~\ref{thm:bayesian}]
  Given a Bayesian network with a constant treewidth, by Proposition~\ref{prop:bayesian}, we can build a polynomial-sized st-d-DNNF of $f$ in polynomial time.
  We can compute the variance of the marginal probability with our proposed algorithm and the built st-d-DNNF of $f$ by method (ii).
  Here, the modification of the procedure around the treatment of $\vnode_{x_i\vert\parval_i}$s are needed.
  By Theorem~\ref{thm:tractable}, the overall running time is polynomial.
\end{proof}

Note that we can prove Theorem~\ref{thm:bayesian} even when under method (i).
For method (i), we should construct the st-d-DNNF of $f\wedge g_\evidence$ instead of $f$.
Observe that $f\wedge g_\evidence$ is the Boolean function obtained by substituting $\lambda_{x_{ij}}$ with $\mathfalse$ for $x_{ij}$ such that $\evidence$ contains $x_{ij^\prime}$ $(j^\prime\neq j)$.
This can be done by replacing each leaf $\lambda_{x_{ij}}$ such that $\evidence$ contains $x_{ij^\prime}$ $(j^\prime\neq j)$ with $\mathfalse$.
The resulting st-d-DNNF obviously respects the vtree constructed in the proof of Proposition~\ref{prop:bayesian} since it is respected by the original st-d-DNNF of $f$.
Moreover, since $f\wedge g_\evidence$ is obtained by assigning some indicator variables of $f$ to $\mathfalse$, Observation~\ref{obs:singleton} holds even if $f$ is replaced with $f\wedge g_\evidence$.
Thus, by eliminating $\mathfalse$ child as described above, we can use the modified version of Algorithm~\ref{alg:prod} that uses (\ref{eq:multicov}) to compute the variance of WMC of $f\wedge g_\evidence$.
Since eliminating $\mathfalse$ child does not increase the st-d-DNNF size, the resulting st-d-DNNF's size is polynomial if the Bayesian network has a constant treewidth.

The tractability of method (i) indicates the possibility of approximated computation for the variance of the conditional probability in a Bayesian network with non-binary random variables in the same way as described in the main text.
Since the conditional probability of $\evidence$ given $\condit$ is $W_{h^\prime}/W_{h}$ with $h^\prime=f\wedge g_\evidence\wedge g_\condit$ and $h=f\wedge g_\condit$, we can approximately compute its variance with the values of $\variance{W_h}$, $\variance{W_{h^\prime}}$, and $\covariance{W_{h^\prime}}{W_h}$, as well as the expectations.
$\covariance{W_{h^\prime}}{W_h}$ can be computed with the modified version of Algorithm~\ref{alg:prod} because Observation~\ref{obs:singleton} holds even when $f$ is replaced with $h$ or $h^\prime$.

\section{Application for Network Reliability}
\label{app:network}
In this appendix, we introduce an application for a network reliability analysis.
We describe a definition of \emph{$K$-terminal network reliability}~\citep{hardy2007knr}, which subsumes two well-known reliability measures: two-terminal network reliability~\citep{aboelfotoh89} and all-terminal network reliability~\citep{won10ATR}.
We are given undirected graph $G=(V,E)$ and probability $p_e$ of the presence of every edge $e\in E$.
We consider a probabilistic graph where each edge $e$ is present with probability $p_e$ and absent with probability $1-p_e$ independent of the other edges' presence or absence.
Given $T\subseteq V$, called a \emph{terminal set}, network reliability $R_T$ is the probability that all the vertices in $T$ are connected, i.e., all the vertices in $T$ are in the same connected component, in the given probabilistic graph.

In an ordinal setting where probability $p_e$ is precisely given, we can compute $R_T$ by WMC as follows.
Let $x_e\ (e\in E)$ be a binary variable where $x_e=\mathtrue$ if $e$ is present and $x_e=\mathfalse$ if $e$ is absent.
Then we consider Boolean function $f_T$ on variable set $\{x_e\mid e\in E\}$ that evaluates to $\mathtrue$ if and only if the vertices in $T$ are connected on a graph $(G,E^\prime)$ where $E^\prime\coloneqq\{e\in E\mid x_e=\mathtrue\}$.
In other words, $f_T$ evaluates to $\mathtrue$ if and only if $T$ is connected in the subgraph induced by the present edges.
In addition, we let $P_{x_e}=1-N_{x_e}=p_e$ for every $e\in E$.
Then the WMC of $f_T$ equals $R_T$.

\citet{nakamura22variance} extended the network reliability evaluation so that every $p_e$ value has uncertainty and the uncertainty degree of $R_T$ is computed.
Particularly, they focused on a situation where $p_e$ is also a random variable with mean $\mu_e$ and variance $\sigma_e^2$ and proposed an algorithm to compute the variance of $R_T$.
Note that $p_e$ and $p_{e^\prime}$ $(e\neq e^\prime)$ are assumed to be independent in~\citep{nakamura22variance} to ensure that each link's presence or absence is independent of the other links' one.
Their algorithm first constructs an OBDD representing $f_T$ by the algorithm of \citet{hardy2007knr} and computes the variance of $R_T$ with the built OBDD.
Here, the size of the built OBDD of $f_T$ is polynomial when the graph's \emph{pathwidth} is constant and their variance computation runs in polynomial time in the size of OBDD.
Thus, their algorithm can compute the variance of the network reliability in polynomial time for a graph with a constant pathwidth.

Although their paper~\citep{nakamura22variance} did not mention WMC, we find that their algorithm can be extended to compute the variance of a general WMC.
Particularly, we can recover their situation by setting $\mu_{P_x}=1-\mu_{N_x}=\mu_e$ and $\sigma_{P_x}^2=\sigma_{N_x}^2=-\sigma_{P_xN_x}=\sigma_e^2$ for $x=x_e\ (e\in E)$.
Here, the assumption that $(P_x,N_x)$ and $(P_y,N_y)$ $(x\neq y)$ are independent is justified by the assumption that $p_e$ and $p_{e^\prime}$ $(e\neq e^\prime)$ are independent.
We also extend their algorithm to handle structured d-DNNFs, which are a strict superset of OBDDs.
Our theoretical contribution for the variance computation of network reliability is summarized in the following theorem.
\begin{theorem}
  \label{thm:networkreliability}
  We are given graph $G$, the mean and the variance of probability $p_e$ of every $e\in E$, and terminals $T\subseteq V$.
  Then we can compute the variance of network reliability $R_T$ in polynomial time when $G$'s treewidth is constant.
\end{theorem}
\begin{proof}
  It is known that the graph property that ``all the vertices in $T$ are connected in the subgraph $(V,E^\prime)$ of $G$'' can be described by a monadic second-order logic (MSO) on graphs.
  By the algorithm of \citet{amarilli17icalp}, we can in polynomial time construct an st-d-DNNF of polynomial size representing $f_T$ when $G$'s treewidth is constant.
  The theorem follows by applying Theorem~\ref{thm:tractable} to their built st-d-DNNF.
\end{proof}
Theorem~\ref{thm:networkreliability} theoretically improves the previous result stating that the variance of network reliability can be computed in polynomial time for constant-pathwidth graphs because the treewidth can be bounded by the pathwidth but not vice versa.

Currently, this approach is less practical because the constant hidden in the complexity of the algorithm of \citet{amarilli17icalp} is generally prohibitively large.
However, there are some practical algorithms for compiling a graph-related property into an st-d-DNNF, e.g., the bag-based search~\citep{ishihata23bag}.
We hope such algorithms can be applied to compile an st-d-DNNF representing $f_T$ and obtain a practical treewidth-bounded algorithm for the variance computation of network reliability.

\section{More Experimental Results}
\label{app:experiment}
\subsection{Detailed Results of Time Consumption}
We describe more detailed results of experiments that measured the time consumption of our proposed algorithm.
Table~\ref{tb:experiment} shows the results for the top-20 networks in descending order of the compiled SDD size.
As mentioned in the main text, the maximum size of SDD was 3,888, and the maximum variance computation time is 0.025 sec.
We roughly observe that the variance computation time increases when the SDD size becomes larger, which meets the complexity result of Theorem~\ref{thm:tractable}.
Even when the time required for compiling SDD is added to the computational time, we can compute the variance of the marginal in at most 10 sec for all the networks.

\begin{table}[tb]
  \centering
  {\footnotesize\tabcolsep=2pt
  \begin{tabular}{lrrrr}
    \toprule
    \multicolumn{1}{c}{Name} & \multicolumn{1}{c}{\#rv} & \multicolumn{1}{c}{$|\text{SDD}|$} & \multicolumn{1}{c}{Compile (s)} & \multicolumn{1}{c}{Variance (s)} \\
    \midrule
    projectmanagement & 26 & 3888 & 0.500 & 0.025 \\
    grounding & 36 & 3397 & 2.387 & 0.017 \\
    GDIpathway2 & 28 & 2755 & 0.784 & 0.021 \\
    gasifier & 40 & 2127 & 1.082 & 0.007 \\
    windturbine & 122 & 2043 & 1.380 & 0.009 \\
    engines & 12 & 1804 & 0.240 & 0.011 \\
    inverters & 29 & 1721 & 0.593 & 0.008 \\
    cng & 86 & 1626 & 1.281 & 0.005 \\
    rainstorm & 34 & 1625 & 1.398 & 0.004 \\
    GDIpathway1 & 28 & 1319 & 0.577 & 0.003 \\
    construnctionproductivity & 18 & 1008 & 0.272 & 0.003 \\
    poultry & 47 & 953 & 0.474 & 0.002 \\
    electrolysis & 16 & 918 & 0.196 & 0.002 \\
    gasexplosion & 18 & 903 & 0.189 & 0.001 \\
    algalactivity2 & 8 & 842 & 0.092 & 0.002 \\
    propellant & 49 & 710 & 9.083 & 0.001 \\
    BOPfailure2 & 46 & 706 & 0.319 & 0.001 \\
    lithium & 45 & 704 & 0.221 & 0.001 \\
    dragline & 86 & 657 & 0.324 & 0.001 \\
    vessel1 & 16 & 644 & 0.148 & 0.001 \\
    \bottomrule
  \end{tabular}
  }
  \caption{Detailed experimental results. \#rv is the number of random variables. $|\text{SDD}|$ is the size of compiled SDD. Compile and Variance indicate the time required to compile SDD and to compute variance.}
  \label{tb:experiment}
\end{table}

We used the SDD package~\citep{choi13sdd} as a compiler, although there exists another st-d-DNNF compiler: miniC2D.
\citet{oztok15minic2d} reported that miniC2D typically produces a much larger st-d-DNNF compared to that built by the SDD package, although its compilation time is often shorter than the SDD package.
They also reported that the size of the st-d-DNNF compiled by miniC2D can be further reduced by the SDD package's size reduction feature after compilation.
This combination of miniC2D and the SDD package enabled us to compile medium-sized SDDs for some Boolean functions such that the sole use of the SDD package takes a prohibitively long time.
Since the compactness of the compiled st-d-DNNF is crucial in the efficiency of the variance computation (Theorem~\ref{thm:tractable}) and all the networks can be compiled in reasonable time with the SDD package, we used the SDD package as a compiler in the experiments.
However, for more complicated networks such that the SDD package cannot compile, we can select the combination of miniC2D and the SDD package for variance computation.

We finally discuss why the SDD size remains small for these networks.
Although some networks have a bunch of parameters, many parameters inside these networks have values of either $0$ or $1$, i.e., many parameters satisfy $\prob(x_{i0}\vert\parval_i)=0$ or $1$.
For example, although the ``propellant'' network has 8,345 parameters, all except 33 have values of either $0$ nor $1$.
Typically, the parameters valued $0$ or $1$ represent conditional branches from expert knowledge, and thus they are determined regardless of the learning from data.
Ace v3.0 suppresses the parameter variables corresponding to such parameters to simplify the resulting Boolean function.
This simplification is also valid for the variance computation because these parameters' variances are $0$, and thus the absence of these parameter variables does not affect the variance of the WMC value.
This is why the compiler outputs smaller SDDs even though the Bayesian networks have many random variables.

\subsection{More Examples of Variance Computation}
We exhibit a few more examples of the variance computation of the marginal probability of a Bayesian network.
We suggest below that we can barely know how much the variance of the marginal is decreased by reducing the variance of the parameter without actually computing the variance.

\begin{figure}[tb]
  \centering
  \includegraphics[keepaspectratio]{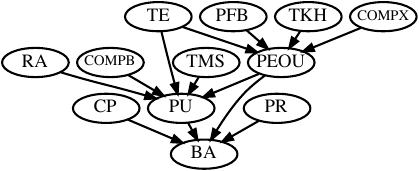}
  \caption{The ``blockchain'' network.}
  \label{fig:blockchain}
\end{figure}
\begin{table}[tbp]
  \centering
  {\footnotesize\tabcolsep=2pt
  \begin{tabular}{lr}
    \toprule
    \multicolumn{1}{c}{Parameter} & \multicolumn{1}{c}{Variance} \\
    \midrule
    $\text{PR}$ & 0.001069 \\
    $\text{BA}\vert\text{CP}_\text{High},\text{PEOU}_\text{Low},\text{PR}_\text{High},\text{PU}_\text{High}$ & 0.001236 \\
    $\text{CP}$ & 0.001243 \\
    $\text{BA}\vert\text{CP}_\text{High},\text{PEOU}_\text{High},\text{PR}_\text{High},\text{PU}_\text{High}$ & 0.001374 \\
    $\text{BA}\vert\text{CP}_\text{High},\text{PEOU}_\text{Low},\text{PR}_\text{Low},\text{PU}_\text{High}$ & 0.001410 \\
    $\text{BA}\vert\text{CP}_\text{High},\text{PEOU}_\text{Low},\text{PR}_\text{High},\text{PU}_\text{Low}$ & 0.001414 \\
    $\text{BA}\vert\text{CP}_\text{High},\text{PEOU}_\text{High},\text{PR}_\text{Low},\text{PU}_\text{High}$ & 0.001424 \\
    $\text{BA}\vert\text{CP}_\text{High},\text{PEOU}_\text{Low},\text{PR}_\text{Low},\text{PU}_\text{Low}$ & 0.001432 \\
    $\text{BA}\vert\text{CP}_\text{Low},\text{PEOU}_\text{High},\text{PR}_\text{High},\text{PU}_\text{High}$ & 0.001442 \\
    $\text{TMS}$ & 0.001454 \\
    \midrule
    (none) & 0.001482 \\
    \bottomrule
  \end{tabular}
  }
  \caption{Variance of $\prob(\text{BA}\!=\!\text{Low})$ when one parameter's variance is reduced to one-tenth. Top-10 parameters in ascending order of variance are shown.}
  \label{tb:showcase2}
\end{table}

One example is the ``blockchain'' network from bnRep, shown in Fig.~\ref{fig:blockchain}.
This network has 12 random variables, each of which is valued either Low or High.
With the same setting of variances, i.e., $\sigma^2=p(1-p)/10$ for the probability parameters with value $p$, the mean and variance of $\prob(\text{BA}=\text{Low})$ are computed as 0.8985 and 0.001482.
Then, we demonstrated how much the variance of the marginal is decreased by reducing the variance of one parameter to one-tenth.
The ``blockchain'' network has 48 parameters whose value is neither $0$ nor $1$, and we conducted the above analysis for each one.
Table~\ref{tb:showcase2} shows the top-10 parameters in the reduction of the variance of $\prob(\text{BA}=\text{Low})$.
Unlike the results for the ``algalactivity2'' network in the main text, most of the parameters in Table~\ref{tb:showcase2} are either the conditional probability of $\text{BA}$ or the probability of the random variables on which $\text{BA}$ directly depends ($\text{PR}$ and $\text{CP}$).
Although this result follows our intuition, we already observed with the ``algalactivity2'' network that it is not always the case.
Thus, we here assert that we cannot obtain parameters having greater impact on the variance of the marginal and evaluate the degree of such an impact without actually computing the variances.

\begin{figure}[tb]
  \centering
    \includegraphics[keepaspectratio]{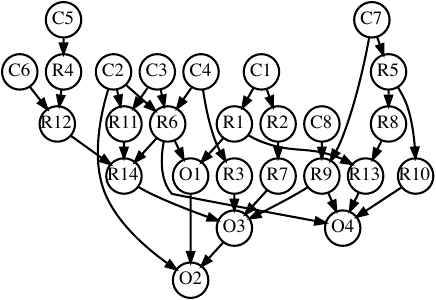}
    \caption{The ``projectmanagement'' network.}
    \label{fig:projectmanagement}
\end{figure}
\begin{figure}
  \begin{minipage}{0.48\columnwidth}
      \centering
      {\footnotesize\tabcolsep=2pt
      \begin{tabular}{lr}
        \toprule
        \multicolumn{1}{c}{Parameter} & \multicolumn{1}{c}{Variance} \\
        \midrule
        $\text{O2}\vert\text{C2}_\text{NO},\text{O1}_\text{NO},\text{O3}_\text{YES}$ & 0.006075 \\
        $\text{C2}$ & 0.006753 \\
        $\text{O2}\vert\text{C2}_\text{YES},\text{O1}_\text{NO},\text{O3}_\text{YES}$ & 0.006809 \\
        $\text{O2}\vert\text{C2}_\text{YES},\text{O1}_\text{YES},\text{O3}_\text{YES}$ & 0.007523 \\
        $\text{O1}\vert\text{R1}_\text{NO},\text{R6}_\text{YES}$ & 0.007796 \\
        $\text{C1}$ & 0.007859 \\
        $\text{R1}\vert\text{C1}_\text{YES}$ & 0.007869 \\
        $\text{O1}\vert\text{R1}_\text{NO},\text{R6}_\text{NO}$ & 0.008209 \\
        $\text{C4}$ & 0.008222 \\
        $\text{R1}\vert\text{C1}_\text{NO}$ & 0.008227 \\
        \midrule
        (none) & 0.008313 \\
        \bottomrule
      \end{tabular}
      }
      \captionof{table}{Variance of $\prob(\text{O2}\!=\!\text{YES})$ when one parameter's variance is reduced to one-tenth. Top-10 parameters in ascending order of variance are shown.}
      \label{tb:showcase3}
  \end{minipage}%
  \hspace{0.03\columnwidth}
  \begin{minipage}{0.48\columnwidth}
      \centering
      {\footnotesize\tabcolsep=2pt
      \begin{tabular}{lr}
        \toprule
        \multicolumn{1}{c}{Parameter} & \multicolumn{1}{c}{Variance} \\
        \midrule
        
        $\text{C7}$ & 0.003732 \\
        $\text{C8}$ & 0.004033 \\
        $\text{O4}\vert\text{R10}_\text{NO},\text{R13}_\text{NO},\text{R6}_\text{YES},\text{R9}_\text{NO}$ & 0.004176 \\
        $\text{C2}$ & 0.004203 \\
        $\text{R10}\vert\text{R5}_\text{NO}$ & 0.004253 \\
        $\text{R9}\vert\text{C7}_\text{NO},\text{C8}_\text{YES}$ & 0.004317 \\
        $\text{C3}$ & 0.004339 \\
        $\text{C4}$ & 0.004394 \\
        $\text{O4}\vert\text{R10}_\text{YES},\text{R13}_\text{NO},\text{R6}_\text{YES},\text{R9}_\text{NO}$ & 0.004416 \\
        $\text{O4}\vert\text{R10}_\text{YES},\text{R13}_\text{NO},\text{R6}_\text{NO},\text{R9}_\text{NO}$ & 0.004418 \\
        \midrule
        (none) & 0.004547 \\
        \bottomrule
      \end{tabular}
      }
      \captionof{table}{Variance of $\prob(\text{O4}\!=\!\text{YES})$ when one parameter's variance is reduced to one-tenth. Top-10 parameters in ascending order of variance are shown.}
      \label{tb:showcase4}
  \end{minipage}%
\end{figure}

Another example is the ``projectmanagement'' network from bnRep, shown in Fig.~\ref{fig:projectmanagement}.
This network has 26 random variables, each of which is valued either YES or NO, and 100 independent parameters.
We conducted the same analyses on two marginal probabilities, $\prob(\text{O2}=\text{YES})$ and $\prob(\text{O4}=\text{YES})$.
When every parameter's variance is set to $p(1-p)/10$, the mean and variance of $\prob(\text{O2}=\text{YES})$ are 0.4493 and 0.008313 and those of $\prob(\text{O4}=\text{YES})$ are 0.3337 and 0.004547.
Tables~\ref{tb:showcase3} and \ref{tb:showcase4} list the top-10 parameters in the reduction of the variance of $\prob(\text{O2}=\text{YES})$ and $\prob(\text{O4}=\text{YES})$.
We observe that the parameters in Tables~\ref{tb:showcase3} and \ref{tb:showcase4} are substantially different.
Notably, although both O2 and O4 are dependent on C1 and C3, C1 only appears in Table~\ref{tb:showcase3}, while C3 only appears in Table~\ref{tb:showcase4}.
These results again indicate that we cannot find these parameters without actually computing the variance of the marginal.

\end{document}